%% file: main.tex
\documentclass[conference]{IEEEtran}

\usepackage{cite}
\usepackage{amsmath,amssymb,amsfonts}
\usepackage{threeparttable}

\usepackage{makecell}

\usepackage{algorithm}
\usepackage{algpseudocode}
\usepackage{xcolor}
\usepackage{amsthm}

\usepackage{graphicx}
\usepackage{textcomp}
\usepackage{mleftright}
\usepackage{dsfont}
\usepackage{mathrsfs}
\usepackage{upgreek}

\usepackage{graphicx}
\usepackage{subfigure}    

\DeclareMathOperator*{\argmax}{argmax}

\newcommand{\prob}{\mathop{\text{P}}}     
\newcommand{\E}{\mathop{\mathbb{E}}}

\newcommand{\1}{\mathds{1}}   

\def\BibTeX{{\rm B\kern-.05em{\sc i\kern-.025em b}\kern-.08em
    T\kern-.1667em\lower.7ex\hbox{E}\kern-.125emX}}

\newtheorem{theorem}{Theorem}
\newtheorem{corollary}{Corollary}
\newtheorem{lemma}{Lemma}

\theoremstyle{definition}
\newtheorem{definition}{Definition}
\newtheorem{remark}{Remark}

\newenvironment{hproof}{%
  \proof}{\endproof}

\newcommand{\final}[1]{#1}

\begin{document}

\title{On the Robustness of Age for Learning-Based Wireless Scheduling in Unknown Environments}


\author{\IEEEauthorblockN{Juaren Steiger\quad\quad Bin Li}
\IEEEauthorblockA{Department of Electrical Engineering, The Pennsylvania State University, University Park, PA, USA
\\Email: \{juaren.steiger, binli\}@psu.edu}
}

\maketitle

\begingroup
  \renewcommand\thefootnote{}
  \footnotetext{This work was supported in part by NSF grant CNS-2152658 and the ARO grant W911NF-24-1-0103.}
\endgroup


\begin{abstract}
The constrained combinatorial multi-armed bandit model has been widely employed to solve problems in wireless networking and related areas, including the problem of wireless scheduling for throughput optimization under unknown channel conditions. Most work in this area uses an algorithm design strategy that combines a bandit learning algorithm with the virtual queue technique to track the throughput constraint violation. These algorithms seek to minimize the virtual queue length in their algorithm design. However, in networks where channel conditions change abruptly, the resulting constraints may become infeasible, leading to unbounded growth in virtual queue lengths.
In this paper, we make the key observation that the dynamics of the head-of-line age, i.e. the age of the oldest packet in the virtual queue, make it more robust when used in algorithm design compared to the virtual queue length. We therefore design a learning-based scheduling policy that uses the head-of-line age in place of the virtual queue length. We show that our policy matches state-of-the-art performance under i.i.d. network conditions. Crucially, we also show that the system remains stable even under abrupt changes in channel conditions and can rapidly recover from periods of constraint infeasibility. 
\end{abstract}


\section{Introduction}

The \textit{combinatorial multi-armed bandit} (CMAB) problem has been widely adopted to model modern engineering problems involving the online optimization of some time-varying quantity with unknown statistics. In recent years, \textit{constrained bandit} models, where the selection of each arm is constrained over time, have been applied to problems in wireless networking  and related areas (e.g. \cite{li2019, steiger2023,huang2024,steiger2024,wu2024, wang2023neural}). In particular, CMAB with \textit{fairness constraints}, where each arm must be selected a minimum fraction of the time, have been utilized to model wireless scheduling with 
\final{minimum throughput or scheduling constraints}
(e.g. \cite{li2019,wu2024,wu2023rateadaptation}). \final{For example, in \cite{wu2024},} 
the cumulative reward represents the cumulative network throughput, and the unknown reward statistics represent the unknown wireless channel statistics. The fairness constraints mirror \textit{quality of service} constraints to meet per-link minimum throughput requirements. 

As a concrete example, consider the agricultural remote monitoring system pictured in Fig.~\ref{fig:remote_monitoring} in which a central controller receives information updates from $K$ wireless sensing sources (a.k.a. nodes). In each timeslot $t$, the controller must select a feasible scheduling action $A_t$, which schedules a subset of the $K$ nodes to transmit their information updates. But, the network state $S_t$ at time $t$, which captures the channel conditions of the wireless link between each node $k$ and the controller, is unknown when scheduling. If the transmission from node $k$ to the controller under scheduling action $A_t$ is successful in timeslot $t$, the controller receives $R_k(S_t, A_t) = 1$ as feedback, and $R_k(S_t, A_t) = 0$ otherwise. In the language of CMAB, each node $k$ is a \textit{bandit arm}, each scheduling action $A_t$ is a \textit{combinatorial super-arm}, and $R_k(S_t, A_t)$ is the \textit{reward}, which also depends on the unknown network state $S_t$. The objective of the controller is to maximize the total network throughput $\sum_{t=1}^T \sum_{k=1}^K R_k(S_t, A_t)$. In order to get an accurate picture of the environment that fairly incorporates information from all sensing sources, the controller would also like to guarantee that the \textit{short-term throughput} over a fixed window $W$ exceeds some requirement $\chi_k$ for each node $k$, i.e. $\frac{1}{W}\sum_{\tau = t- W+1}^{t} R_k(S_\tau,A_\tau) \geq \chi_k$.

\begin{figure}
    \centering
    \includegraphics[width=1\columnwidth]{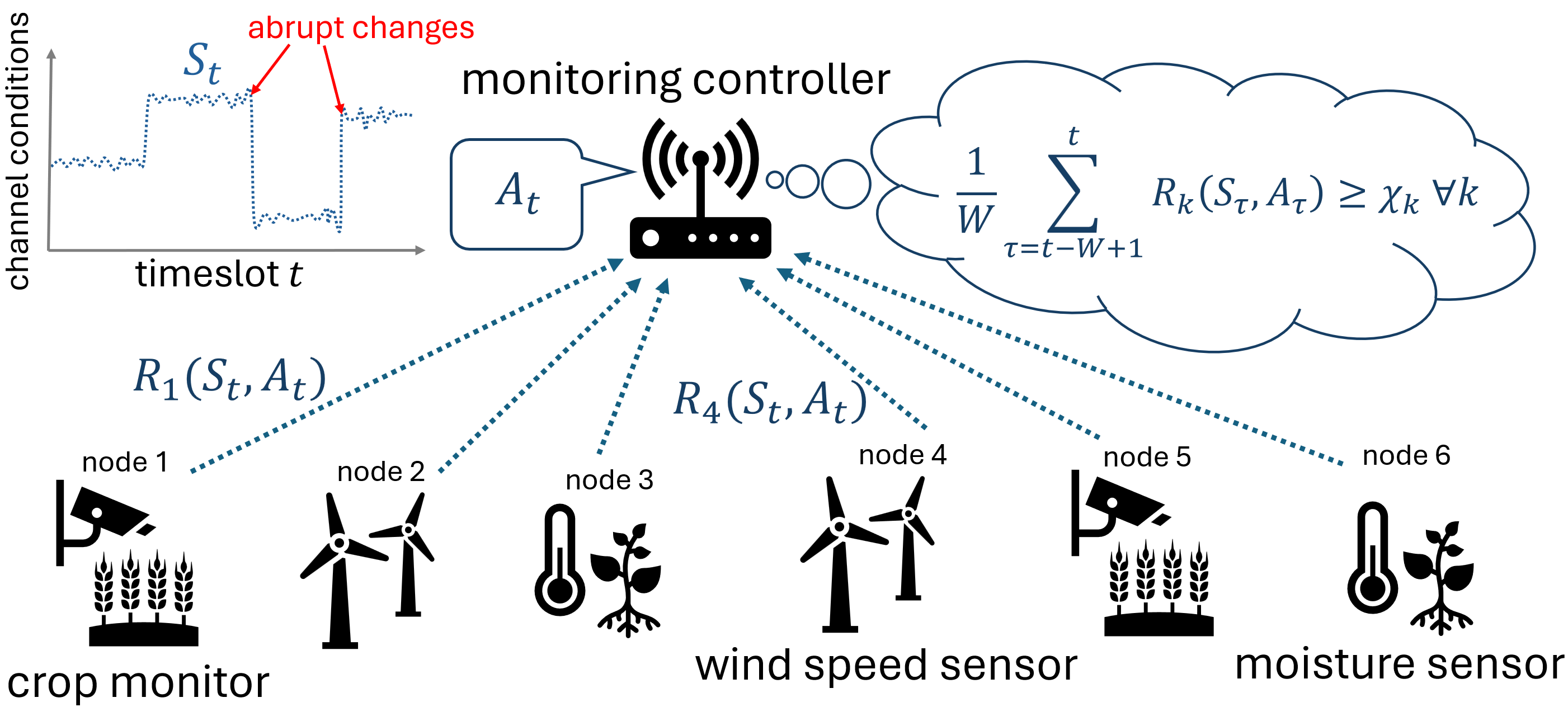}
    \caption{An agricultural remote monitoring system with short-term throughput requirements and abruptly changing wireless channel conditions.}
    \label{fig:remote_monitoring}
\end{figure}

Existing work on constrained CMAB problems primarily uses the \textit{virtual queue} technique from the field of \textit{stochastic network optimization} \cite{neely2010} to track the cumulative constraint violation, where \textit{virtual packets} are added and removed from the queue depending on whether the violation is increased or decreased. To our knowledge, all the prior work using this approach (e.g. \cite{li2019,liu2021, steiger2023,steiger2024,wu2024,wang2023neural,guo2023rectified}) use the \textit{pessimistic-optimistic} algorithm design \cite{liu2021} that combines a bandit learning algorithm (e.g. \textit{Upper Confidence Bound (UCB)}) with the virtual queues. Here, the algorithm selects the action with the maximal \textit{UCB-plus-virtual-queue-length} in each timeslot. State-of-the-art results in this direction show zero constraint violation after some fixed timeslot \cite{liu2021} under i.i.d. system parameters, i.e. $S_t$ in our example.  

However, the network state $S_t$ is often subject to abrupt changes in real wireless networks that violate this i.i.d. assumption. For example, the agricultural remote monitoring system may experience periods of extreme weather, or congestion due to community events where many people visit the farm. As another example, consider military tactical grids \cite{don2020information}, where  connected devices operate in harsh unknown environments and are subject to network resource contention with adversaries and jamming attacks. In such nonstationary and adversarial environments, the throughput constraint of a link may become temporarily infeasible. Then the queue length, whose increment is dominated by the virtual packet arrivals, will blow up and the queue-length-based algorithm will continue trying to schedule that link, leading to the starvation of the other links. 

In this paper, we make the crucial observation that using the \textit{head-of-line age}\footnote{This is also called \textit{head-of-line delay} by some authors (e.g. \cite{neely2010}).}, i.e. the age of the oldest virtual packet in the queue, is robust to abrupt disruptions in the usual network state behavior that lead to temporarily infeasible constraints. This is because, unlike the queue length, the head-of-line age always increments by 1 and has a much larger decrement related to the interarrival time between virtual packets, which gives the other queue more of a chance to be scheduled. 

Head-of-line age-based scheduling policies (e.g., \cite{mckeown2002achieving,andrews2004scheduling,shakkottai2002scheduling,eryilmaz2005stable}) have been widely studied in the wireless scheduling literature for their ability to prioritize timely packet delivery. However, their behavior and robustness under overloaded network conditions (e.g., \cite{georgiadis2006optimal,li2013optimal}) remain insufficiently understood. In parallel, while there is growing interest in incorporating online learning into scheduling design, particularly in unknown environments, the integration of age-based metrics with learning-based approaches has yet to be systematically explored. 

We also note that in recent years, the integration of \textit{Age-of-Information} (AoI) with bandit learning algorithms has been explored; \cite{li2021} was the first to consider this via the \textit{time-since-last-play} of arms. AoI is also an important metric for information update systems, such as the remote monitoring system in Fig.~\ref{fig:remote_monitoring}, as it can be used to measure \textit{data freshness} \cite{sun2019age}. Since the work  \cite{li2021}, others have considered bandit problems with AoI (e.g. \cite{huang2024informationgathering}). While AoI does not suffer from the aforementioned problem with the virtual queue length in abruptly changing environments, it cannot provide guarantees on short-term throughput like the virtual queue technique can. Inspired by wireless scheduling algorithms that combine both the virtual queue length and AoI (e.g. \cite{lu2018agebased}), some recent work had applied this to the constrained CMAB setting (e.g. \cite{wu2024,wang2024nextword}). However, as we show in this paper, these algorithms still suffer from the use of the virtual queue length in abruptly changing environments. It can also be shown that the head-of-line age can provide guarantees for AoI-type metrics, which adds more motivation to study it in an online learning context. 

With all that in mind, we outline our main contributions and the organization of this paper as follows:

(i) In Section~\ref{section:system_model}, we describe the wireless networking system model as a constrained CMAB problem. A minor contribution here is that we consider the throughput constraints as a moving average over a fixed window $W$, whereas prior work (e.g. \cite{huang2024,steiger2024,wu2024}) considers the cumulative moving average up to the current timeslot $t$. 

(ii) In Section~\ref{section:age}, we describe the virtual queue and head-of-line age dynamics, and show that the throughput constraint violation can be bounded in terms of the head-of-line age (Lemma~\ref{lemma:safety_rate_lower_bound} and Theorem~\ref{theorem:age_is_all_you_need}) in an i.i.d. environment. This motivates the design of an \textit{age-based bandit learning policy} (Algorithm~\ref{alg:age_based}), which combines the UCB algorithm and associated tuning parameter $\eta$ with the head-of-line age and associated tuning parameter $\varepsilon$. To our knowledge, this is both the first time the head-of-line age has been considered for virtual packets 
\final{generated according to some constraint violation instead of real packets (e.g. as in \cite{joo2019heavytraffic})}
, and also the first time the head-of-line age has been incorporated into an online learning algorithm. 

(iii) In Section~\ref{section:performance}, we show the empirical performance of our algorithm against three state-of-the-art policies from the literature and prove its theoretical performance bounds in an i.i.d. environment that matches state-of-the-art (e.g. \cite{wu2024}). Namely, we show that the algorithm achieves zero constraint violation if the window size $W = O(\eta/ \varepsilon)$ (Theorem~\ref{theorem:safety}), and achieves the regret bound $O(\varepsilon T + T/\eta + \sqrt{T\log T})$ (Theorem~\ref{theorem:regret}). We note that these results rely on a more complicated Lyapunov drift lemma (Lemma~\ref{lemma:drift_lemma}) than all the prior work using queue lengths can use (e.g. Lemma 11 in \cite{liu2021}) since the drift of the head-of-line age is not absolutely bounded by a constant. We also note in Remark~\ref{remark:tslr_guarantee} that the age-based policy can give an AoI (a.k.a. \textit{time-since-last-reward} (TSLR) in the bandit context) guarantee of $O(\eta)$. 

(iv) In Section~\ref{section:robustness_of_age}, we empirically investigate what happens when the i.i.d. assumption is violated, and demonstrate our key observation that the head-of-line age is more robust under periods of temporary constraint infeasibility compared to queue-length-based approaches. Another interesting result here is that the age-based policy maintains stable TSLR under temporarily infeasible constraints, even compared to algorithms that attempt to optimize this metric (e.g. \cite{wu2024}). We also include experiments using a real data trace in this section. Finally, in Section~\ref{section:conclusion}, we conclude the paper. 

\textbf{Notation}: To simplify notation, we assume $\sup \varnothing = \inf \varnothing = 0$ and $\sum_{i=a}^b x_i = 0$, $\bigcup_{i=a}^b \mathcal{X}_i = \varnothing$, etc. for $a > b$. We consider the Geometric distribution with support $\{1,2,\ldots\}$. $\Delta x_t = x_{t+1} - x_t$ is the forward difference operator. $[n] \triangleq \{ 1,2,\ldots , n\}$ for integer $n$. $f(x) = O(g(x))$ means $\limsup_{x \to \infty} \frac{f(x)}{g(x)} < \infty$ for positive functions $f$ and $g$. $x\land y \triangleq \min\{ x,y \}$ for two real numbers $x$ and $y$.

\final{
\textbf{Proofs}: All omitted proofs of theoretical results can be found in the Appendix.
}

\section{System Model}
\label{section:system_model}

We consider a central controller that receives information updates wirelessly  from $K$ nodes (e.g. the remote monitoring system in  Fig.~\ref{fig:remote_monitoring}). Time is slotted up to a time horizon $T$ and in each timeslot $t\in [T]$, the state of the network is given by $S_t \in \mathcal{S}$ where $\mathcal{S}$ is the set of possible network states. Here, the network state captures the channel conditions of all the wireless links between the $K$ nodes and the controller. We use $X_k(s) = 1$ to indicate that the controller can successfully receive a packet from node $k$ when the network is in state $s \in \mathcal{S}$, and $X_k(s) = 0$ otherwise.\footnote{The analysis in this paper can be easily extended to the \textit{``bursty Bernoulli"} case where either 0 or $C_k$ packets are transmitted from node $k$ in timeslot $t$.} We assume the network state $S_t$ is not observable by probing, and is therefore treated as an unknown random variable. In each timeslot $t$, the controller selects a scheduling action $A_t \in \mathcal{A}$ from the set of feasible scheduling actions $\mathcal{A}$, which captures any interference constraints between nodes that may be known to the controller a priori. We use $I_k(a) = 1$ to indicate that node $k$ is scheduled under a scheduling action $a \in \mathcal{A}$ and $I_k(a) = 0$ otherwise.
\final{
Then $R_k(s,a) \triangleq X_k(s)I_k(a)$ indicates that the controller receives an update packet from node $k$ under scheduling action $a$ when the network is in state $s$.
}
By the end of timeslot $t$, the controller therefore receives the feedback $(R_k(S_t,A_t))_{k=1}^K$. Then the history of observations available to the player at the beginning of timeslot $t$ is $\mathcal{H}_t \triangleq \left(A_{\tau}, (R_{k}(S_\tau, A_\tau))_{k=1}^K \right)_{\tau=1}^{t-1}$. Let $A_t^\pi \in \mathcal{A}$ denote the action chosen in timeslot $t$ under policy $\pi$. We consider randomized policies, where a policy $\pi$ is a collection of probability mass functions over $\mathcal{A}$, i.e.  $\pi \triangleq {\{ \pi_{t} : \mathcal{A} \to [0,1] \}_{t=1}^T}$. Then ${A_t^\pi \sim \pi_{t}}$ in each timeslot $t$. If the policy under consideration is obvious, we omit the superscript for brevity. Next, we define the class of policies which can be employed by a non-omniscient controller, called \textit{causal policies}.
\begin{definition}[causal policy]
\label{definition:causal_policy}
A policy $\pi$ is called \textit{causal} if for each timeslot $t$, the action $A_t^\pi \sim {\pi_{t}}$ is independent from $S_\tau$ for all $\tau \geq t$. 
\end{definition}
According to Definition~\ref{definition:causal_policy}, a policy that depends only on the history $\mathcal{H}_t$ and other random variables generated independently in each timeslot for the operation of the algorithm (\final{e.g., the virtual packet arrivals to be introduced later}) is causal. 

In the language of combinatorial multi-armed bandits (CMAB), each node $k$ along with its wireless link to the controller is a \textit{bandit arm}, $A_t$ is an \textit{action}, and $R_k(S_t,A_t)$ is the reward received from arm $k$ in timeslot $t$ when action $A_t$ is played. Here, we assume the network state $S_t$ is i.i.d. over timeslots, although we will violate this assumption in Section~\ref{section:robustness_of_age}. Therefore for each node $k$, $X_k(S_t)$ has a common mean $\overline{x}_k$, i.e. the \textit{channel rate} $\overline{x}_k \triangleq \E\mleft[ X_k(S_t) \mright]$ for all timeslots $t$. The channel rates $(\overline{x}_k)_{k=1}^K$ are also unknown a priori due to the unknown network state dynamics. We are interested in finding a causal policy $\pi$ that maximizes the cumulative expected reward, i.e. the cumulative expected system throughput $\sum_{t=1}^T \sum_{k=1}^K \E\mleft[ R_k(S_t,A_t^\pi) \mright]$.

We are also interested in guaranteeing that the \textit{short-term throughput} for each node $k$ exceeds a given requirement $\chi_k \in [\chi_{\min},1]$, where $\chi_{\min} > 0$. This is important in remote monitoring systems, for example, to get an accurate real-time picture of the environment that fairly incorporates information from all sensing sources. The short-term throughput over a fixed window $W$ for node $k$ under a policy $\pi$ is given by
\begin{equation}
\overline{R}_{k,t}^\pi[W] \triangleq  \frac{1}{W}\sum_{\tau = t- W+1}^{t} R_k(S_\tau,A_\tau^\pi).
\end{equation}
To simplify the analysis, we also assume $T$ is a multiple of $W$, although our results can be extended to the case when this condition is removed.

Our goal is then to learn a causal policy $\pi$ that solves the following constrained online learning problem:
\begin{equation}
\label{eq:main_problem}
\begin{aligned}
\max_{\text{causal }\pi } \quad &\sum_{k=1}^K \sum_{t=1}^T \E\mleft[R_{k}(S_t, A_t^\pi)\mright] \\
\text{subject to} \quad &\E\mleft[\overline{R}_{k,t}^\pi[W]\mright] \geq \chi_k \quad  \forall \,k \in [K] ,\,t \geq W  \\
\end{aligned}
\end{equation}
under a priori unknown channel rates $(\overline{x}_{k})_{k=1}^K$.

Let $\pi^*$ denote the causal policy that optimizes the problem \eqref{eq:main_problem}. Next, we define the notion of a  \textit{static policy} and show that there exists a static policy $\sigma^*$ that matches the performance of $\pi^*$ in problem \eqref{eq:main_problem}. We therefore define regret according to this static policy. 

\begin{definition}[static policy]
A causal policy $\sigma$ is a \textit{static} policy if ${\sigma_{t}(a)} = \sigma_{1}(a)$ for all timeslots $t \in [T]$ and actions $a \in \mathcal{A}$. For static policies, we omit subscript $t$ and simply write $\sigma(a)$.
\end{definition}


Let $\sigma^*$ denote the \textit{optimal static policy} that maximizes the objective in problem \eqref{eq:main_problem}, and which is defined by $\sigma^*(a) \triangleq \frac1T \sum_{t=1}^T \pi^*_{t}(a)$ for each  action $a$. To simplify notation, we denote the action played under the optimal static policy by $A_t^*$, i.e. $A_t^* \sim \sigma^*$. Part (i) in the following lemma, which follows standard analysis in stochastic network optimization \cite{neely2010}, guarantees the existence of $\sigma^*$:

\begin{lemma}
\label{lemma:existence_of_static_policy}
For any causal policy $\pi$ that is a feasible solution to \eqref{eq:main_problem}, there is a corresponding static policy $\sigma$ defined by $\sigma(a) = \frac1T \sum_{t=1}^T \pi_{t}(a)$ for all actions ${a \in \mathcal{A}}$ that achieves
\begin{equation}
\text{(i)}\quad T \sum_{k=1}^K \E\mleft[ R_{k}(S_\tau,A_\tau^\sigma) \mright]
= \sum_{t=1}^T \sum_{k=1}^K \E\mleft[ R_{k}(S_t,A_t^\pi) \mright]
\end{equation}
\final{for all timeslots $\tau$}, and (ii) $\E\mleft[  R_{k}(S_\tau,A_\tau^\sigma) \mright] \geq \chi_k$ for all $k \in [K]$ and timeslots $\tau$.    
\end{lemma}

\begin{proof}
Fix arm $k$ and timeslot $\tau$. We have 
\begin{equation*}
\begin{aligned}
\E\mleft[ I_k(A_\tau^\sigma) \mright]
&= \textstyle \sum_{a \in \mathcal{A}} I_k(a) \,\sigma(a) = \sum_{a \in \mathcal{A}} I_k(a) \frac1T \sum_{t=1}^T \pi_{t}(a) \\
&\hspace{-0.2cm}= \textstyle \frac1T \sum_{t=1}^T  \sum_{a \in \mathcal{A}} I_k(a) \pi_{t}(a) = \frac1T \sum_{t=1}^T \E\mleft[ I_k(A_t^\pi) \mright].    
\end{aligned}
\end{equation*}
Since the static policy $\sigma$ is also a causal policy, $S_\tau$ and $A_\tau^\sigma$ are independent. Then $\E\mleft[ R_{k}(S_\tau,A_\tau^\sigma) \mright] = \E\mleft[X_k(S_\tau)\mright] \E\mleft[I_k(A_\tau^\sigma)\mright]$ and from the previous manipulations, we have
\begin{equation*}
\hspace{-0.2cm}
\begin{aligned}
&\textstyle\E\mleft[ R_{k}(S_\tau,A_\tau^\sigma) \mright] = \E\mleft[X_k(S_\tau)\mright] \frac{1}{T}\sum_{t=1}^T \E\mleft[ I_k(A_t^\pi) \mright] \\
&\textstyle\overset{(a)}{=} \frac{1}{T}\sum_{t=1}^T \E\mleft[X_k(S_t)\mright] \E\mleft[ I_k(A_t^\pi) \mright] 
\overset{(b)}{=} \frac{1}{T}\sum_{t=1}^T \E\mleft[R_k(S_t,A_t^\pi)\mright]  \\
\end{aligned}
\end{equation*}
where $(a)$ is because $S_\tau$ is i.i.d. over $\tau$ and $(b)$ is because $\pi$ is a causal policy and therefore each $S_t$ and $A_t^\pi$ are independent. Summing over $k\in [K]$ and multiplying by $T$ gives part (i).

For part (ii), fix arm $k$ and timeslot $\tau$. From the above manipulations, it suffices to show that $\E\mleft[\overline{R}_{k,t}^\pi[W]\mright] \geq \chi_k$ for all $t \geq W$ implies that $\frac{1}{T}\sum_{t=1}^T\E\mleft[ R_{k}(S_t,A_t^\pi) \mright] \geq \chi_k$. Recall that it was assumed $T$ is a multiple of $W$. Then 
\begin{equation*}
\begin{aligned}
&\textstyle\frac{1}{T}\sum_{t=1}^T\E\mleft[ R_{k}(S_t,A_t^\pi) \mright] \\
&= \textstyle\frac{W}{T} \sum_{j=1}^{T/W}\frac{1}{W} \sum_{t = (j-1)W + 1}^{jW}\E\mleft[ R_{k}(S_t,A_t^\pi) \mright] \\
&= \textstyle\frac{W}{T} \sum_{j=1}^{T/W} \E\mleft[ \overline{R}_{k,jW}^\pi[W] \mright] \geq \frac{W}{T} \sum_{j=1}^{T/W} \chi_k = \chi_k,
\end{aligned}
\end{equation*}
and the result in part (ii) follows.
\end{proof}

In light of Lemma~\ref{lemma:existence_of_static_policy}, we define the \textit{regret} of a policy $\pi$ by comparing it to the performance of $\sigma^*$ up to time $T$:
\begin{equation}
\mathrm{Reg}^\pi_T \triangleq \sum_{k=1}^K \sum_{t=1}^T \E\mleft[ R_{k}(S_t, A_t^{*}) - R_{k}(S_t, A_t^\pi) \mright].
\end{equation}

\final{
Finally, we make the common assumption of \textit{Slater's condition}\footnote{This assumption may potentially be removed using techniques from \cite{guo2022rectifiedpessimisticoptimisticlearningstochastic}.}, where the throughput constraint is satisfied according to a \textit{tightness} parameter $\gamma \in (0,1]$ by a ``$\gamma$-tight'' static policy $\sigma^\gamma$. Let $A_t^\gamma \sim \sigma^\gamma$ denote the action under the $\gamma$-tight static policy in each timeslot $t$, which is assumed to satisfy
\begin{equation}
\label{eq:gamma_tight_static_policy}
\begin{aligned}
\overline{x}_k\E\mleft[I_k(A_t^\gamma) \mright] 
\geq \chi_k + \gamma \quad\forall k \in [K].
\end{aligned}
\end{equation}

}

\section{Age-Based Bandit Learning Algorithm Design}
\label{section:age}

For each arm $k$, consider a virtual queue that stores \textit{``delivery requests"} that must be addressed by choosing actions that result in a real packet delivery for the node (arm) $k$. Here, a delivery request is a virtual packet in the virtual queue. At the beginning of each timeslot $t$, a delivery request is generated with probability $\chi_k + \varepsilon$, where $0 < \varepsilon \leq \gamma / 2$ is a tuning parameter, and inserted at the tail of the queue. At the end of the timeslot, a delivery request is removed from the head of the queue if the selected action results in a delivery for the arm, i.e. $R_k(S_t,A_t) = 1$. In this section, we show that controlling the age of the oldest delivery request in the queue, which is also the head-of-line delivery request under the first-come-first-served queueing discipline, controls the throughput for the corresponding arm. After detailing the operation of the virtual queue in Section~\ref{section:queue_dynamics}, we first show that this is true for any policy in Section~\ref{section:safety_for_any_policy}. Then, in Section~\ref{section:safety_for_age_based_policies}, we define what it means for a policy to be \textit{age-based}, and show in Theorem~\ref{theorem:age_is_all_you_need} that for any age-based policy, the throughput constraint violation $\chi_k - \E\mleft[\overline{R}_{k,t}^\pi[W]\mright]$ is bounded by a function involving the expected head-of-line age, the window size $W$, and the tuning parameter $\varepsilon$. This result motivates the design of an age-based bandit learning policy, which we introduce in Section~\ref{section:age_based_policy}.

\subsection{The Queue of delivery requests and the Head-of-Line Age}
\label{section:queue_dynamics}
The pair $(k,\ell)$ signifies the $\ell$-th delivery request generated for arm $k$ and we use $\uptau_k(\ell)$ to denote the timeslot in which it was generated. For ease of exposition, we also set $\uptau_k(0) \triangleq 0$. Then the interarrival time between two consecutive delivery requests is geometrically distributed, i.e. for all $\ell$ we have
\begin{equation}
\label{eq:interarrival}
\Delta \uptau_k(\ell) \triangleq \uptau_k(\ell+1) - \uptau_k(\ell) \sim \text{Geometric}(\chi_k + \varepsilon).
\end{equation}
Similarly, let $d_k(\ell)$ denote the timeslot in which delivery request $(k,\ell)$ departs the queue. We denote the set of delivery requests in the queue for arm $k$ at the beginning of timeslot $t$ (before the delivery request for timeslot $t$ is generated) by 
\begin{equation}
\label{eq:queue_def}
\begin{aligned}
\mathcal{Q}_{k,t} 
&\triangleq 
\underbrace{\left\{ \ell : \uptau_k(\ell) < t  \right\}}_{\text{arrivals before }t} \setminus \underbrace{\left\{ \ell : d_k(\ell) < t  \right\}}_{\text{departures before }t}.
\end{aligned}
\end{equation}

\begin{figure}
    \centering
    \includegraphics[width=1\columnwidth]{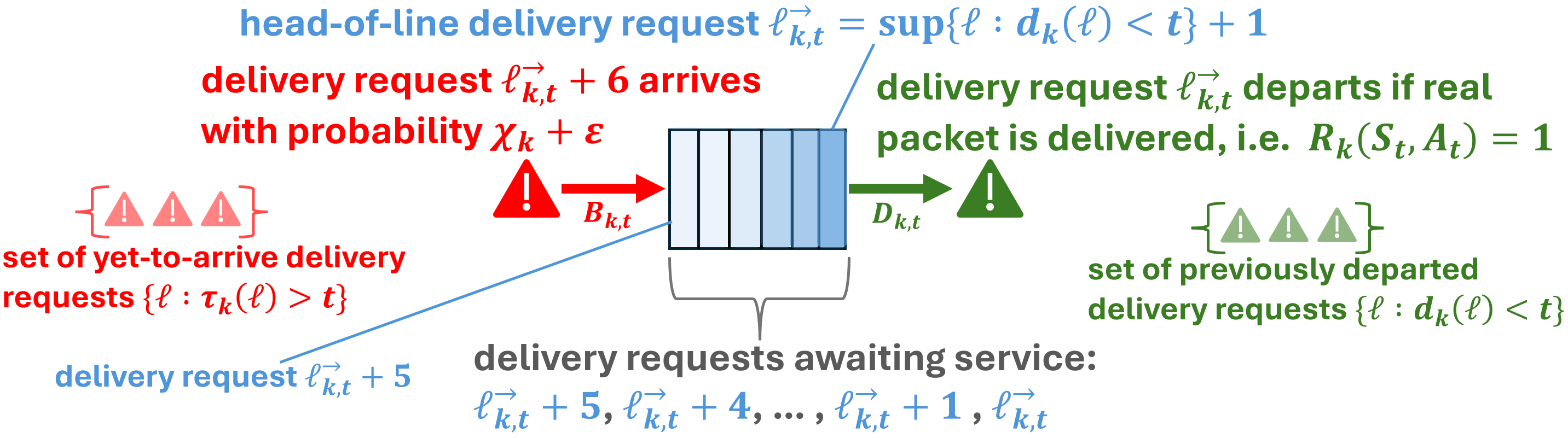}
    \caption{The virtual queue $\mathcal{Q}_{k,t}$ of delivery requests for arm $k$ in timeslot $t$.}
    \label{fig:queue}
\end{figure}

Each queue follows the \textit{first-come-first-served} queueing discipline, meaning that for $\ell' < \ell$, delivery request $(k,\ell')$ is served and departs the virtual queue for arm $k$ before delivery request $(k,\ell)$. Then since $\{ \ell : d_k(\ell) < t \}$ denotes the set of delivery requests that have already departed the queue by timeslot $t$, under first-come-first-served queueing, $\sup\{ \ell : d_k(\ell) < t \}$ is the last departed delivery request. It follows that the \textit{next-to-depart} delivery request is given by \footnote{We use the arrow $\to$ notation in the superscript, i.e. $(\,\cdot\,)^\to$, to convey that a quantity is in some way related to the next-to-depart virtual packet.}
\begin{equation}
\label{eq:next_to_depart}
\ell^\to_{k,t} \triangleq \underbrace{\sup\{ \ell : d_k(\ell) < t \}}_{\text{last departed delivery request}} + \,1.   
\end{equation}
If $\mathcal{Q}_{k,t} = \varnothing$, then $\ell^\to_{k,t}$ will be the next-to-arrive delivery request, or if $\mathcal{Q}_{k,t} \neq \varnothing$, then $\ell^\to_{k,t}$ will be the head-of-line delivery request. Note that $\uptau_{k}(\ell_{k,t}^\to) \leq t$ means that the next-to-depart delivery request was either already in the queue by the beginning of timeslot $t$ or it arrives during timeslot $t$. Then a departure can only occur in timeslot $t$ if $\uptau_{k}(\ell_{k,t}^\to) \leq t$ and we indicate that a departure does occur, i.e. $\ell_{k,t}^\to$ departs, by
\begin{equation}
\label{eq:queue_departure}
D_{k,t} \triangleq |\{ \ell : d_k(\ell) = t \}| = R_k(S_t,A_t)\,\1\{ \uptau_{k}(\ell_{k,t}^\to) \leq t \}.
\end{equation}
We also indicate that a delivery request arrives to the queue in timeslot $t$ using the notation
\begin{equation}
\label{eq:queue_arrival}
B_{k,t} \triangleq |\{ \ell : \uptau_k(\ell) = t \}|  \sim \text{Bernoulli}(\chi_k + \varepsilon).  
\end{equation}
The dynamics of the queue of delivery requests are summarized in Fig.~\ref{fig:queue}. The age of delivery request $(k,\ell)$ in timeslot $t$ is given by
\begin{equation}
\label{eq:age_def}
Z_{k,t}(\ell) \triangleq t - \uptau_k(\ell).
\end{equation}
When the age is positive, it represents the number of timeslots since the packet was generated. However, note that we also allow the age to take negative values, which represents the number of timeslots until the packet will be generated. While this will prove theoretically useful, in reality, the player can only observe $[Z_{k,t}(\ell)]^+$ in timeslot $t$. The age of the head-of-line delivery request, called the \textit{head-of-line age}, is
\begin{equation}
\label{eq:head_of_line_age}
Z_{k,t}^\to \triangleq [Z_{k,t}(\ell^\to_{k,t})]^+.
\end{equation}
The following lemma gives the dynamics of the head-of-line age, 
which is similar to observations in prior work (e.g. \cite{neely2013delaybased,joo2019heavytraffic}), and 
which will be useful in the subsequent results. 
\begin{lemma}
\label{lemma:head_of_line_age}
\final{For any arm $k$,} if the next-to-depart delivery request \final{in timeslot $t$} arrives  by $t$, i.e. $\uptau_k(\ell_{k,t}^\to) \leq t$, then  the drift of the head-of-line age is
\begin{equation}
\Delta Z_{k,t}^\to = 1 - D_{k,t} \min\left\{ t + 1 - \uptau_k(\ell_{k,t}^\to), \Delta \uptau_k(\ell_{k,t}^\to) \right\}.
\end{equation}
Otherwise, if  $\uptau_k(\ell_{k,t}^\to) > t$, then $\Delta Z_{k,t}^\to = 0$. 
\end{lemma}

\begin{hproof}
Fix \final{arm $k$ and timeslot $t$}. According to the definition of age \eqref{eq:age_def}, we have for two consecutive delivery requests that their age difference satisfies $Z_{k,t}(\ell) -  Z_{k,t}(\ell+1) = \Delta \uptau_{k}(\ell)$. Also from the definition \eqref{eq:age_def}, we have 
\final{$Z_{k,t+1}(\ell+1) = Z_{k,t}(\ell+1) + 1$}.
Combining these two observations gives that $Z_{k,t+1}(\ell+1) = Z_{k,t}(\ell) - \Delta \uptau_{k}(\ell) + 1$. 
\final{
Also, note that according to the definition of age \eqref{eq:age_def}, we have $\left[Z_{k,t}(\ell_{k,t}^\to)\right]^- = \left[\uptau_k(\ell_{k,t}^\to) -  t \right]^+$. Then we can rewrite the definition of head-of-line age \eqref{eq:head_of_line_age}  as $Z_{k,t}^\to = Z_{k,t}(\ell_{k,t}^\to) + \left[\uptau_k(\ell_{k,t}^\to) -  t \right]^+$.
}
We use \final{the previously stated identities}
to analyze two cases:

\textit{Case 1}: $D_{k,t} = 1$. Then note that $\ell_{k,t+1}^\to = \ell_{k,t}^\to + 1$.

\textit{Case 2}: $D_{k,t} = 0$. Then note that $\ell_{k,t+1}^\to = \ell_{k,t}^\to$.

\noindent The detailed analysis of these two cases is omitted due to space limitations. 
\end{hproof}

\subsection{Throughput for Any Policy in Terms of Head-of-Line Age}
\label{section:safety_for_any_policy}

In this section, we give throughput guarantees for each arm $k$ under any policy in terms of the head-of-line age $Z_{k,t}^\to$. 
\begin{lemma}
\label{lemma:safety_rate_lower_bound}
For any arm $k$ and timeslot $t \geq W$, under any policy, the throughput constraint violation is bounded by
\begin{equation}
\label{eq:frequency_violation_lemma_rewrite}
\textstyle\chi_k - \E\mleft[\overline{R}_{k,t}[W] \mright] \leq \frac{1}{W} \E\mleft[1 + \sum_{\tau = 1}^{Z_{k,t+1}^\to - 1} B_{k,t-\tau + 1} \mright] - \varepsilon.
\end{equation}
\end{lemma}

\begin{proof}
Fix the timeslot $t$ and arm $k$. Note that, by definition \eqref{eq:next_to_depart}, ${\{ \ell : d_k(\ell) < t \}} = \{ \ell : \ell < \ell_{k,t}^\to \}$. Therefore from \eqref{eq:queue_def},
\begin{equation}
\begin{aligned}
\label{eq:prelim_queue_decomp}
\hspace{-0.2cm}\mathcal{Q}_{k,t}
&= \left\{ \ell : \uptau_k(\ell) < t  \right\} \setminus \left\{ \ell : \ell < \ell_{k,t}^\to  \right\} 
= \left\{ \ell \geq \ell_{k,t}^\to : \uptau_k(\ell) < t  \right\} \\
&\hspace{-0.2cm}= \left\{ \ell : \uptau_k(\ell_{k,t}^\to) \leq \uptau_k(\ell) < t \right\} 
= \textstyle\bigcup_{\tau = \uptau_k(\ell_{k,t}^\to)}^{t-1} \{ \ell : \uptau_k(\ell) = \tau \}. \\
\end{aligned}
\end{equation}
We can also simply rewrite \eqref{eq:queue_def} by decomposing each set into $t-1$ disjoint subsets, i.e. $\mathcal{Q}_{k,t} = {\bigcup_{\tau=1}^{t-1} \{ \ell : \uptau_k(\ell) = \tau \}} \setminus {\bigcup_{\tau=1}^{t-1} \{ \ell : d_{k}(\ell) = \tau \}}$. It follows from the definitions \eqref{eq:queue_departure} and \eqref{eq:queue_arrival} that $|\mathcal{Q}_{k,t}| = \sum_{\tau=1}^{t-1} B_{k,\tau} - \sum_{\tau=1}^{t-1} D_{k,\tau}$,
and from \eqref{eq:prelim_queue_decomp} that 
\final{$|\mathcal{Q}_{k,t}| = \sum_{\tau = \uptau_k(\ell_{k,t}^\to)}^{t-1} B_{k,\tau}$.}

 Since these results hold for all timeslots $t$, fix $t \geq W$. Then
\begin{equation}
\begin{aligned}
\label{eq:next_lower_bound_in_sample_path_lemma}
&\textstyle \sum_{\tau = \uptau_k(\ell_{k,t+1}^\to)}^{t} B_{k,\tau} = \sum_{\tau=1}^t [B_{k,\tau} - D_{k,\tau}]  \\
&= \textstyle \sum_{\tau=1}^{t-W} [B_{k,\tau} - D_{k,\tau}] + \sum_{\tau=t-W+1}^{t} [B_{k,\tau} - D_{k,\tau}] \\
&\overset{(a)}{\geq} \textstyle \sum_{\tau=t-W+1}^{t} (B_{k,\tau} - R_{k}(S_\tau, A_\tau)) 
\end{aligned}
\end{equation}
where $(a)$ bounds $\sum_{\tau=1}^{t-W} [B_{k,\tau} - D_{k,\tau}] = |\mathcal{Q}_{k,t-W+1}| \geq 0$ and $D_{k,\tau} \leq R_{k}(S_\tau, A_\tau)$ (see \eqref{eq:queue_departure}).
Note that $B_{k,\tau} = 1$ by definition for $\tau = \uptau_k(\ell_{k,t+1}^\to)$. Then the left-hand side of the above can be bounded by $\sum_{\tau = \uptau_k(\ell_{k,t+1}^\to)}^{t} B_{k,\tau} \leq 1 + \sum_{\tau = 1}^{t - \uptau_k(\ell_{k,t+1}^\to)} B_{k,t-\tau + 1}$. Combining this with \eqref{eq:next_lower_bound_in_sample_path_lemma}, dividing by $W$, taking expectations, and rearranging gives
\begin{equation}
\textstyle\chi_k - \E\mleft[\overline{R}_{k,t}[W] \mright] \leq \frac{1}{W} \E\mleft[1 + \sum_{\tau = 1}^{t - \uptau_k(\ell_{k,t+1}^\to)} B_{k,t-\tau + 1} \mright] - \varepsilon .
\end{equation}
Assume $\uptau_k(\ell_{k,t+1}^\to) < t$ (otherwise the sum above would trivially equal zero). Then by definition of the age \eqref{eq:age_def}, the head-of-line age $Z_{k,t+1}^\to = t + 1 -  \uptau_k(\ell_{k,t+1}^\to)$ and the final result follows. 
\end{proof}

\subsection{Throughput Violation Bounds for Age-Based Policies}
\label{section:safety_for_age_based_policies}

In this section, we formalize the idea of an \textit{age-based policy}, and present Theorem~\ref{theorem:age_is_all_you_need}, which shows that under an age-based policy, the throughput constraint violation can be written as a function of the expected head-of-line age, the window size $W$, and the tuning parameter $\varepsilon$.

\begin{definition}[age-based policy]
\label{defition:age_based_policy}
Define the \textit{age-augmented history} as $\mathcal{H}_t^\to \triangleq (\mathcal{H}_t , (Z_{k,\tau}^\to)_{k\in [K], \tau \in [t]})$. An \textit{age-based policy} $\pi$ is a causal policy under which $A_t^\pi$ is a measurable function of $\mathcal{H}_t^\to$ in each timeslot $t$.
\end{definition}

One of the most important features of an age-based policy, namely that the policy plays independently from arrivals after the head-of-line delivery request,  is discussed as follows.

\begin{remark}[arrival independence of age-based policies]
\label{remark:age_based_arrival_indepdendence}
Fix an arm $k$ and timeslot $t$. According to Lemma~\ref{lemma:head_of_line_age} and the definition \eqref{eq:queue_departure} of $D_{k,\tau}$, for any timeslot $\tau$, $\Delta Z_{k,\tau}^\to$ is a measurable function of $(\Delta \uptau_k(\ell))_{\ell=0}^{\ell_{k,\tau}^\to} $ and $ R_{k}(S_\tau, A_\tau)$ (since $\uptau_k(\ell_{k,\tau}^\to) = \sum_{\ell=0}^{\ell_{k,\tau}^\to - 1} \Delta \uptau_k(\ell)$). Then since $Z_{k,t}^\to = \sum_{\tau=1}^{t-1} \Delta Z_{k,\tau}^\to$, the age-augmented history $\mathcal{H}_t^\to$ is a measurable function of $\mathcal{H}_t$ and $\big((\Delta \uptau_k(\ell))_{\ell=0}^{\ell_{k,t-1}^\to}\big)_{k=1}^K$. Under an age-based policy $\pi$, and due to the i.i.d. arrivals, the age-augmented history $\mathcal{H}_t^\to$ is therefore independent of the future arrivals $((B_{k,\tau})_{\tau > \uptau_k(\ell_{k,t}^\to)})_{k=1}^K$ in timeslots after the arrival times of the head-of-line packets, and is independent of the interarrival times  $((\Delta \uptau_k(\ell))_{\ell \geq \ell_{k,t}^\to})_{k=1}^K$ of future packets after the head-of-line packets. 
\end{remark}
Note that the above remark is not necessarily true for queue-length-based policies since the queue lengths $(|\mathcal{Q}_{k,t}|)_{k=1}^K$ also depend on the arrivals $\big((B_{k,\tau})_{\tau = \uptau_k(\ell_{k,t}^\to)+1}^{t-1}\big)_{k=1}^K$ that came after the head-of-line packets. From Lemma~\ref{lemma:safety_rate_lower_bound}, we derive the throughput constraint violation bound under an age-based policy in the following theorem by leveraging Remark~\ref{remark:age_based_arrival_indepdendence} in order to apply Wald's identity.
\begin{theorem}
\label{theorem:age_is_all_you_need}
For any arm $k$, under an age-based policy, the throughput constraint violation is bounded for all timeslots $t \geq W$ by
\begin{equation}
\chi_k - \E\mleft[\overline{R}_{k,t}[W] \mright] \leq \frac{1 + \left(\E\mleft[Z_{k,t+1}^\to\mright] - 1\right)(\chi_k + \varepsilon) }{W}  - \varepsilon.
\end{equation}
\end{theorem}

\begin{proof}
Fix an arm $k$ and timeslot $t$. According to Remark~\ref{remark:age_based_arrival_indepdendence}, under an age-based policy, $Z_{k,t+1}^\to$ is independent from $B_{k,\tau}$ for $\tau > \uptau_k(\ell_{k,t+1}^\to)$. Then since the sum in \eqref{eq:frequency_violation_lemma_rewrite} consists only of terms $B_{k,\tau}$ for $\tau > \uptau_k(\ell_{k,t+1}^\to)$, we can apply Wald's identity to Lemma~\ref{lemma:safety_rate_lower_bound} to get the result.
\end{proof}

The following corollary reveals the tradeoff between the tuning parameter $\varepsilon$ and the expected head-of-line age $\E\mleft[Z_{k,t+1}^\to\mright]$ to obtain a feasible window size for the throughput constraints. 

\begin{corollary}
\label{corollary:age_is_all_you_need}
For any arm $k$, any age-based policy achieves zero throughput constraint violation, i.e. $\chi_k - \E\mleft[\overline{R}_{k,t}[W] \mright] \leq 0$ for all $t \geq W$ if the window size is at least $W \geq \frac{1}{\varepsilon}\left( (\chi_k + \varepsilon)\E\mleft[Z_{k,t+1}^\to\mright] + 1 \right)$.
\end{corollary}

\subsection{An Age-Based Bandit Learning Policy}
\label{section:age_based_policy}

Motivated by Theorem~\ref{theorem:age_is_all_you_need}, we design a very natural age-based bandit learning policy. In particular, we design a \textit{``UCB-plus-MaxWeight"} style policy, which is commonly used for queue-length-based constrained bandit learning algorithms (e.g. \cite{li2019,steiger2022,liu2021}), but substitute the head-of-line age for the queue length. In each timeslot $t$, the policy plays
\final{
\begin{equation}
\label{eq:age_based_algorithm_def}
A_t \in \argmax_{a \in \mathcal{A}}\sum_{k=1}^K \left( \eta U_{k,t} + Z_{k,t}^\to\right)I_k(a) 
\end{equation}
}
where $U_{k,t}$ is the UCB estimate of the channel rate $\overline{x}_k$ in timeslot $t$ and $\eta > 0$ is a tuning parameter that balances reward and \final{constraint violation}. The details are shown in Algorithm~\ref{alg:age_based}. 
This mirrors the \textit{``drift-plus-penalty"} style algorithm design from the field of stochastic network optimization and facilitates Lyapunov drift analysis to bound the instantaneous head-of-line age, which may be plugged in to Theorem~\ref{theorem:age_is_all_you_need} to obtain \final{constraint violation} bounds (Theorem~\ref{theorem:safety}). Also, the \textit{``drift-plus-regret"} analysis technique \cite{liu2021} may be used to obtain a regret bound (Theorem~\ref{theorem:regret}). We derive both of these in Section~\ref{section:performance}.

\begin{algorithm}
\caption{Age-Based Bandit Learning}\label{alg:age_based}
\begin{algorithmic}[1]
\Require $\eta > 0$,  $\varepsilon \in (0,\gamma / 2]$, and $(\chi_k)_{k=1}^K$. 

\State Initialize $N_{k,1} \gets \widetilde{x}_{k,1} \gets 0$ and $\mathcal{Q}_{k,1} \gets \varnothing$  $\forall\,k\in [K]$.
\For {$t \in [T]$}
    \For {$k \in [K]$} \Comment{Compute UCB estimates}
        \State \label{line:definition_ucb_estimate} $U_{k,t} \gets \begin{cases}
         \min\mleft\{ 1, \widetilde{x}_{k,t} + \sqrt{\frac{3 \log t}{2 N_{k,t}}}  \mright\}   & N_{k,t} > 0 \\
         1 & N_{k,t} = 0.
        \end{cases}$
    \EndFor
    \State Retrieve $(Z_{k,t}^\to)_{k=1}^K$ from the virtual queues $(\mathcal{Q}_{k,t})_{k=1}^K$.
    \State \final{Play action $\displaystyle A_t \in \argmax_{a \in \mathcal{A}}\sum_{k=1}^K \left( \eta  U_{k,t} +  Z_{k,t}^\to\right)I_k(a)$.}
    \State Observe rewards $(R_{k}(S_t, A_t))_{k=1}^K$.
    \For {$k \in [K]$} \Comment{Update algorithm variables}
        \If {\final{$I_k(A_t) = 1$}}
            \State $N_{k,t+1} \gets N_{k,t} + 1$.
            \State \final{$\widetilde{x}_{k,t+1} \gets \frac{1}{N_{k,t} + 1} \left[ N_{k,t}\,\widetilde{x}_{k,t} + R_{k}(S_t, A_t) \right]$}.
        \Else 
            \State $(N_{k,t+1}\,,\, \widetilde{x}_{k,t+1})\gets  (N_{k,t}\,,\,\widetilde{x}_{k,t})$.
        \EndIf
        \State Update $\mathcal{Q}_{k,t}$ to $\mathcal{Q}_{k,t+1}$ according to Section~\ref{section:queue_dynamics}.
    \EndFor
\EndFor
\end{algorithmic}
\end{algorithm}

\section{Performance of the Age-Based Bandit Learning Policy in an I.I.D. Environment}
\label{section:performance}

In this section, we analyze the performance of the age-based bandit learning policy given in Algorithm~\ref{alg:age_based}. In Theorem~\ref{theorem:safety}, we derive the minimum window size of $O(\eta/\varepsilon)$ for which the policy achieves zero violation of the throughput constraints and in Theorem~\ref{theorem:regret}, we derive its ${O(\varepsilon T + T/\eta + \sqrt{T\log T})}$ regret bound. Table~\ref{tab:results_summary} summarizes its performance in terms of $T$ for different values of $\eta$ and $\varepsilon$.  Before presenting the detailed theoretical analysis, we show the empirical performance of our policy compared to some state-of-the-art policies that exhibit tradeoffs in reward and constraint violation. 

\begin{table}[t]
\caption{Age-Based Bandit Learning Policy Performance  in Terms of $T$}
\begin{center}
\begin{tabular}{|c|c|c|c|}
\hline
$\eta \propto$ & $\varepsilon \propto$ & \makecell{Zero Throughput \\ Violation Window Size} & Regret \\
\hline \hline
$\eta$ & $\varepsilon$ & $O(\eta/\varepsilon)$ & $O(\varepsilon T + T/\eta + \sqrt{T\log T})$ \\
\hline
$T^{\delta_1}$ & $ T^{-\delta_2}$  & $W = O(T^{\delta_1 + \delta_2})$  & $O(T^{1- (\delta_1 \land \delta_2)})$ \\
\hline
$1$ & $1$ & $W = O(1)$ & $O(T)$ \\
\hline
$\sqrt{T}$ & $1/\sqrt{T}$ & $W = O(T)$ & $O(\sqrt{T\log T})$ \\
\hline
\end{tabular}
\begin{tablenotes}
\item $\delta_1, \delta_2 \in (0,1)$ and  $\delta_1 \land \delta_2 < 1/2$. 
\end{tablenotes}
\end{center}
\label{tab:results_summary}
\end{table}

\subsection{Empirical Performance}
\label{section:empirical_performance}

\final{
In this section, we compare the age-based learning policy \eqref{eq:age_based_algorithm_def} against three different  approaches from the literature.
}
Some of the policies we compare against incorporate the \textit{time-since-last-reward} (TSLR) $\mathbb{T}_{k,t}$ for each arm $k$, which has initial condition $\mathbb{T}_{k,1} = 0$ and is updated from timeslot $t$ to $t+1$ by \final{$\mathbb{T}_{k,t+1} = \begin{cases}
\mathbb{T}_{k,t} + 1 & R_k(S_t,A_t) = 0 \\
0 & R_k(S_t,A_t) = 1 \\
\end{cases}$}.
TSLR is the same as the AoI in the context of wireless scheduling and is often used as a measure of \textit{data freshness}, which is of interest in remote monitoring, and information update systems in general.
\final{
The policies we compare against are as follows: 

(i) A \textit{queue-length-based policy} similar to algorithms found in \cite{li2019} and \cite{liu2021}, which we call  ``QLen'' in simulation figures, and which plays 
${\displaystyle A_t \in \argmax_{a \in \mathcal{A}}\sum_{k=1}^K \left( |\mathcal{Q}_{k,t}| + \eta U_{k,t}\right)I_k(a)}$.

(ii) A \textit{TSLR-based policy} similar to the algorithm found in \cite{li2021}, which we call ``TSLR'' in simulation figures, and which plays $\displaystyle A_t \in \argmax_{a \in \mathcal{A}}\sum_{k=1}^K \left( \mathbb{T}_{k,t} + \eta U_{k,t}\right)I_k(a)$.

(iii) A \textit{combined queue-length-plus-TSLR-based policy} similar to the algorithm found in \cite{wu2024}, which we call {``QLen+TSLR''} in simulation figures, and which plays ${\displaystyle A_t \in \argmax_{a \in \mathcal{A}}\sum_{k=1}^K \left( |\mathcal{Q}_{k,t}| + \alpha \mathbb{T}_{k,t} + \eta U_{k,t}\right)I_k(a)}$. Note that a new tuning parameter $\alpha$ is introduced here for the TSLR term.
}




\final{
For the simulation parameters, we use the same 6-arm setup as the synthetic simulations in \cite{wu2024}. For the tuning parameters, similarly to \cite{wu2024}, we set $\eta = 100$, $\alpha = 1$, and $\varepsilon = 0.001$ for all algorithms. We set the window size to $W = 100$.
}
We plot the simulation results in Fig.~\ref{fig:simulation_results}, including the regret and total throughput constraint violation. We also plot the total TSLR, since it is of interest in information update systems. 

\begin{figure}
  \centering
  \subfigure[Regret]{%
    \includegraphics[width=0.333\columnwidth]{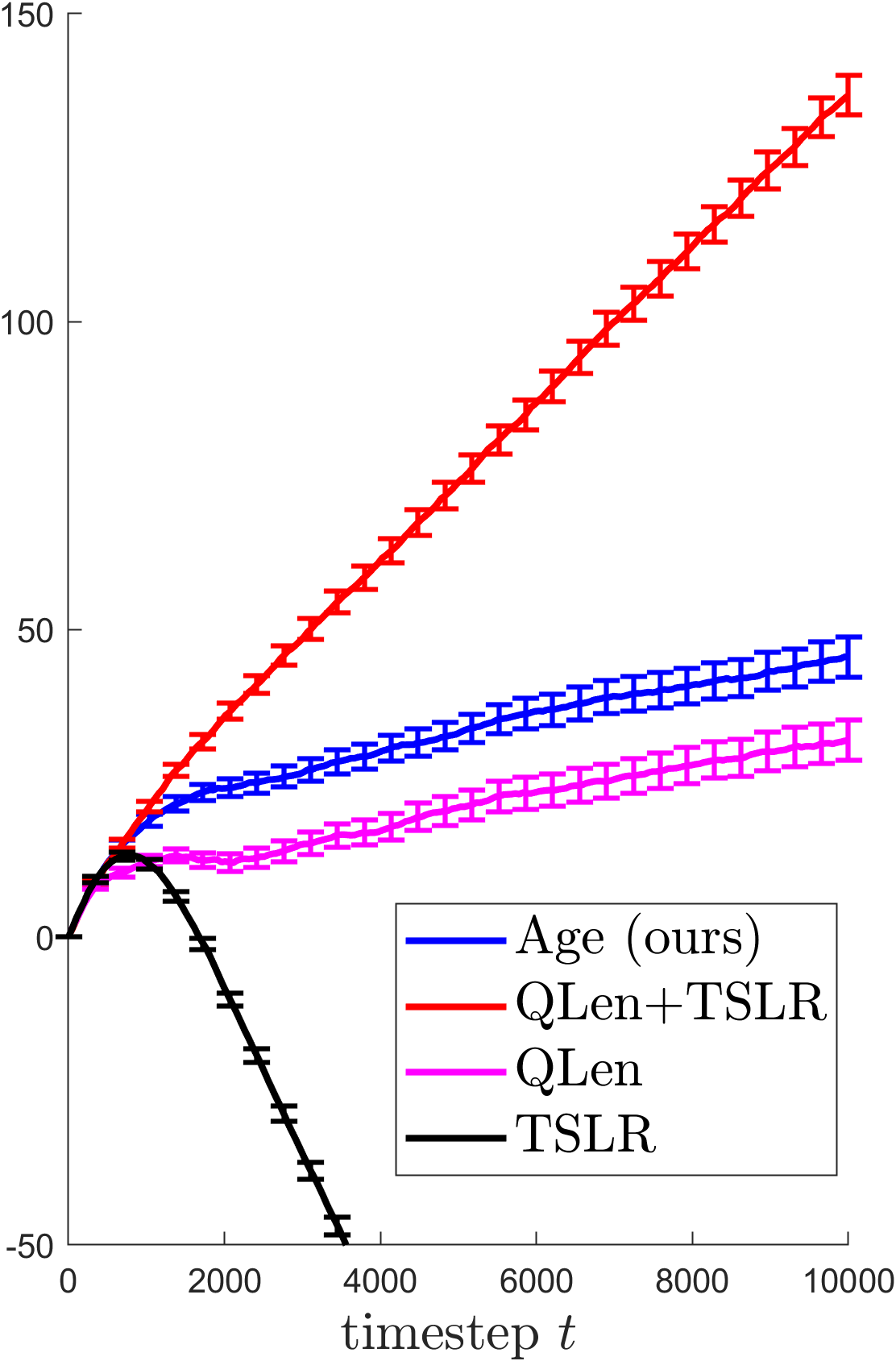}%
    \label{fig:one}}%
  \subfigure[Throughput Violation]{%
    \includegraphics[width=0.333\columnwidth]{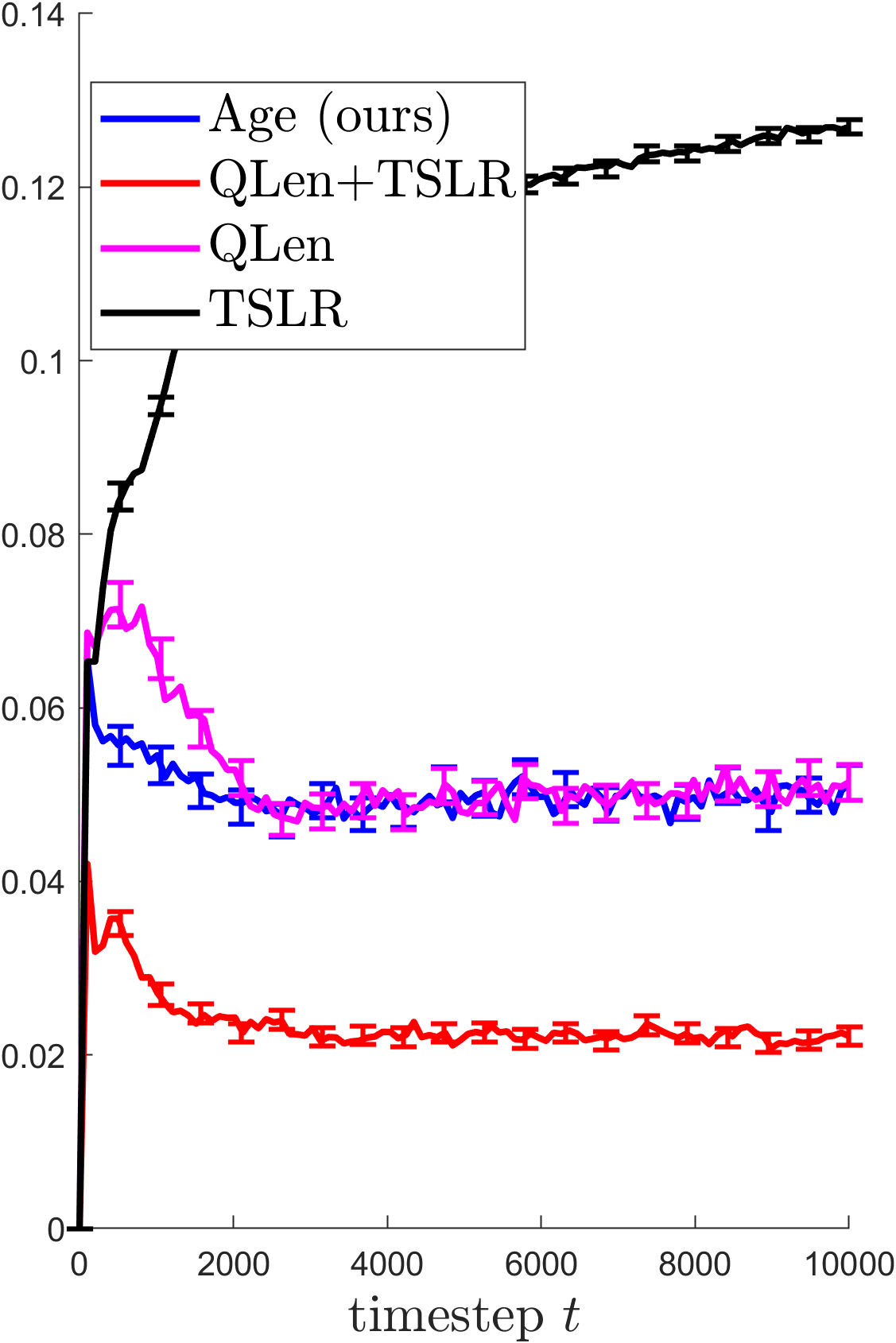}%
    \label{fig:two}}%
  \subfigure[Freshness (TSLR)]{%
    \includegraphics[width=0.333\columnwidth]{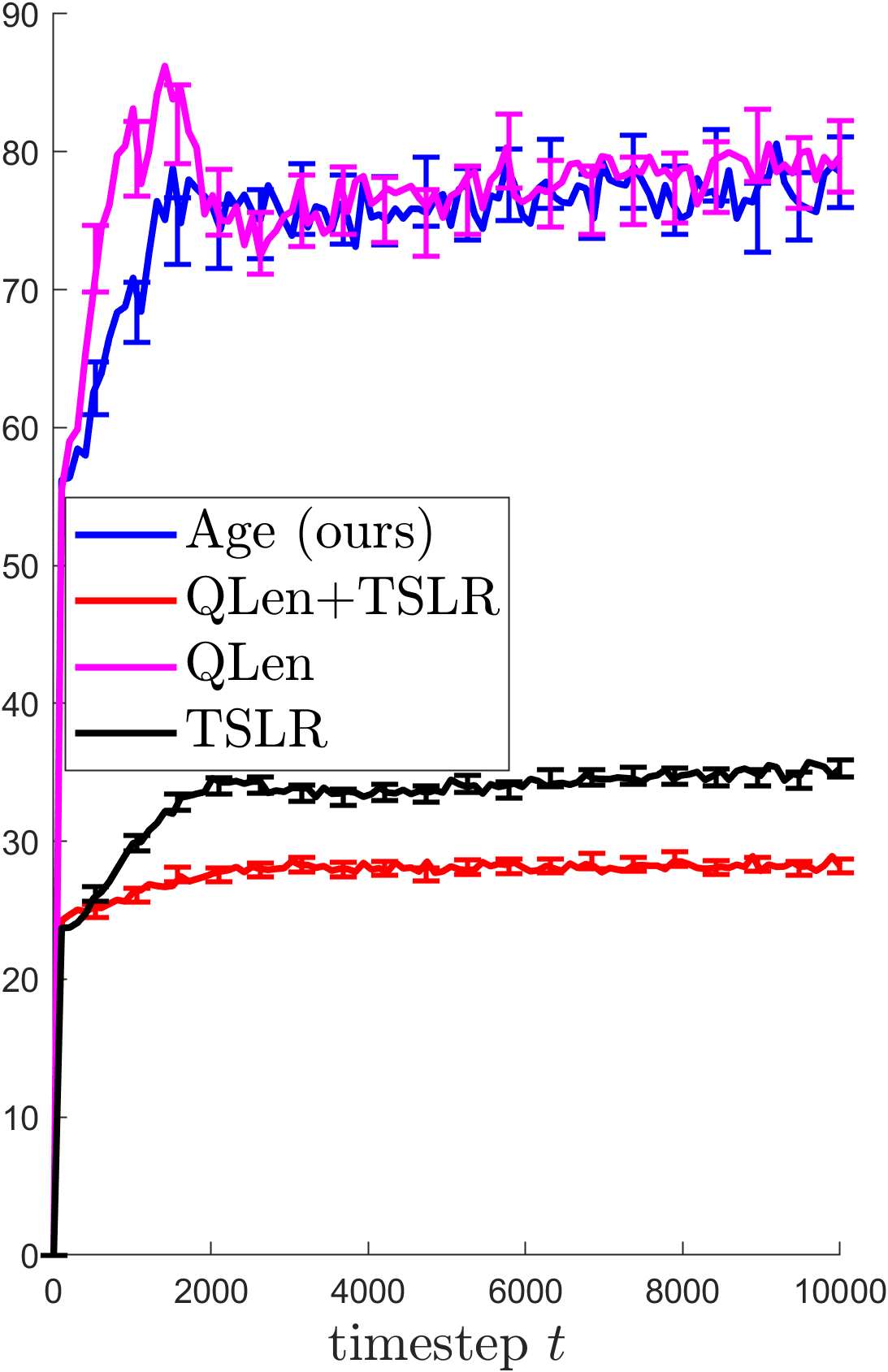}%
    \label{fig:three}}%
  \caption{Result of the simulations from Section~\ref{section:empirical_performance} (1000 simulation runs).}
  \label{fig:simulation_results}
\end{figure}

Fig.~\ref{fig:simulation_results} shows that our age-based policy achieves 
\final{similar performance to the queue-length-based policy in this i.i.d. environment.}
\final{Note that the TSLR-based policy achieves negative regret because it can actually achieve $\widetilde{O}(\sqrt{T} + \frac{T}{\eta})$ regret when compared against the best combinatorial arm in the unconstrained problem \cite{li2021}. However, it fails to achieve the throughput requirements in general. We note that none of the policies considered seem to dominate empirically in all three performance metrics. 
However, in the next section, we show that the age-based policy is superior in the non-i.i.d. case where the channel conditions change abruptly. 
}

\subsection{Theoretical Analysis}

The theoretical results for the age-based bandit learning policy depend heavily on two Lyapunov functions: (i) the \textit{weighted quadratic Lyapunov function} \final{${L_t \triangleq {\frac12 \sum_{k=1}^K \frac{(\chi_k + \varepsilon)}{\overline{x}_k} \left(Z_{k,t}^\to \right)^2}}$}, which is inspired by the Lyapunov function used in \cite{neely2013delaybased} \final{and the Lyapunov function used in \cite{wu2024}}, and 
 (ii) the \textit{Euclidean norm of the weighted head-of-line ages} \final{$\widetilde{L}_t \triangleq \|(\sqrt{(\chi_k + \varepsilon)/\overline{x}_k}Z_{k,t}^\to)_{k=1}^K\|_2$}.
These two Lyapunov functions are connected as follows. Since  $\sqrt{x} - \sqrt{y} \leq \frac{x-y}{2\sqrt{y}}$ for all $x \geq 0$ and $y > 0$, we have 
\begin{equation}
\begin{aligned}
\label{eq:lyapunov_function_connection}
\widetilde{L}_t > 0 \implies \Delta \widetilde{L}_t = \Delta \sqrt{2 L_t} \leq {\Delta L_t}\,\big/\,\widetilde{L}_t.
\end{aligned}
\end{equation}
Most prior work applying Lyapunov drift analysis to constrained bandit learning uses a drift lemma such as Lemma 11 in \cite{liu2021} that operates on a process with the following two drift conditions: (i) if the process exceeds some value at time $t$, then its drift becomes negative at time $t$ in \final{conditional} expectation, and (ii) the magnitude of the drift of the process is bounded almost surely by some constant that does not depend on $t$. The drift lemma says that if these two conditions are satisfied, the moment-generating function (MGF) of the process can be bounded at any time $t$. We would like to apply such a drift lemma to the Lyapunov function $\widetilde{L}_t$. However, note that according to Lemma~\ref{lemma:head_of_line_age}, it will have a decrement that is a function of the Geometric random variables $(\Delta \uptau_k(\ell_{k,t}^\to))_{k=1}^K$ and will therefore violate drift condition (ii). Thankfully, the original work \cite{hajek1982} that this sort of analysis descends from allows for the magnitude of the drift to instead be bounded by a light-tailed random variable. To that end, we instead use the following drift lemma in our analysis, which is derived from the results in \cite{hajek1982}.

\begin{lemma}
\label{lemma:drift_lemma}
Consider a random process $\{ Y_t \}_{t\geq 1}$ and a sequence of $\sigma$-algebras $\{ \mathcal{F}_t \}_{t\geq 1}$ such that each $Y_t$ is $\mathcal{F}_t$-measurable. Assume the following two drift conditions hold:

(i) For all $t$, if $Y_t \geq \varphi$ then $\E\mleft[ \Delta Y_t \mid \mathcal{F}_t \mright] \leq -\rho$, where $\varphi > 0$ and $\rho \in (0,1]$ are constants.

(ii) $\E\mleft[ e^{\theta|\Delta Y_t|} \mid \mathcal{F}_t \mright] \leq M$ for all $t$, where $\theta \in (0,1]$ and $M \in [1,\infty)$ are constants.

Then for all $t$ and any $0 < \zeta \leq \frac{\rho \theta^2}{8 M}$, we have 
\begin{equation}
\textstyle \E\mleft[ e^{\zeta Y_t} \mright] \leq \left(1-\frac{\zeta\rho}{2}\right)^{t-1} \E\mleft[ e^{\zeta Y_1} \mright] + \frac{2}{\zeta\rho}e^{\zeta \varphi} M^{\zeta/\theta}.
\end{equation}
\end{lemma}

We use Lemma~\ref{lemma:drift_lemma} to derive the following instantaneous moment bound on the Lyapunov function $\widetilde{L}_t$.

\begin{lemma}
\label{lemma:lyapunov_function_bound}
Under the age-based learning algorithm, for all timeslots $t \geq 1$, we have 
\final{
\begin{equation}
\E\mleft[ \widetilde{L}_t \mright] \leq  \frac{8\log(4/\gamma)}{\gamma} f\mleft( (\chi_k)_{k=1}^K \mright) + \frac{4}{\gamma}\sum_{k=1}^K \frac{1}{\chi_k \overline{x}_k } + \frac{4K\eta}{\gamma}
\end{equation}
}
where the function $f : (0,1)^K \to (0, \infty)$ is given by
\final{
\begin{equation}
\begin{aligned}
&f(\mathbf{y}) = \min_{0 < \theta \leq h(\mathbf{y})}\left[ \frac{9 g(\mathbf{y},\theta)}{\theta^2}\log\mleft(\frac{32 g(\mathbf{y},\theta)}{\theta^2}\mright) \right] \\
\end{aligned}
\end{equation}
where $h(\mathbf{y}) = -\frac{\overline{x}_{\min}}{2} \log\mleft(1-\min_{k\in [K]}y_k\mright)$, $g(\mathbf{y},\theta) = \prod_{k=1}^K\frac{y_k\, e^{2\theta/\overline{x}_k}}{1 - (1-y_k)e^{2\theta/\overline{x}_k}}$, and $\overline{x}_{\min} \triangleq \min_k \overline{x}_k$.
}
\end{lemma}

\begin{hproof}
We derive the necessary drift conditions to apply Lemma~\ref{lemma:drift_lemma} to $\widetilde{L}_t$. Define $\rho \triangleq \frac{\gamma - \varepsilon}{2}$ and 
\final{
$\varphi \triangleq \frac{1}{\rho}\left(\eta K + \sum_{k=1}^K \frac{1}{\overline{x}_k \chi_k }\right)$.
}
Using Lemma~\ref{lemma:head_of_line_age}, Remark~\ref{remark:age_based_arrival_indepdendence}, and the definition of the age-based bandit learning algorithm \eqref{eq:age_based_algorithm_def}, after much work, we can obtain $\E\mleft[ \Delta L_t \mid \mathcal{H}_t^\to \mright] \leq -2\rho \widetilde{L}_t + \rho \varphi$. Then from \eqref{eq:lyapunov_function_connection}, we obtain drift condition (i): $\widetilde{L}_t > \varphi \implies \E\mleft[ \Delta \widetilde{L}_t \mid \mathcal{H}_t^\to \mright] \leq  -\rho$. For drift condition (ii), define \final{$\theta_{\max} \triangleq -\frac{\overline{x}_{\min}}{2} \log(1-\chi_{\min})$}  and consider $\theta \in (0,\theta_{\max}]$ and \final{$M \triangleq \prod_{k=1}^K \frac{\chi_k\, e^{2\theta/\overline{x}_k}}{1 - (1-\chi_k)e^{2\theta/\overline{x}_k}}$}. Then the drift condition $\E\mleft[ e^{\theta \left| \Delta \widetilde{L}_t \right|} \mid \mathcal{H}_t^\to \mright] \leq M$ can be obtained using the reverse triangle inequality and Lemma~\ref{lemma:head_of_line_age}. Then we apply Lemma~\ref{lemma:drift_lemma} to get the MGF bound, and apply Jensen's inequality and simplify to get the final moment bound. The full proof is omitted due to space limitations.
\end{hproof}
We are now ready to derive the throughput guarantees for the age-based bandit learning policy.

\begin{theorem}
\label{theorem:safety}
The age-based  learning policy \eqref{eq:age_based_algorithm_def} achieves  zero throughput constraint violation, i.e. $\E\mleft[\overline{R}_{k,t}[W] \mright] \geq \chi_k$ for all $t \geq W$ if the window size 
\final{$W$ satisfies}
\final{
\begin{equation}
W \geq  \frac{8\log(4/\gamma)}{\gamma\varepsilon} f\mleft( (\chi_k)_{k=1}^K \mright) + \frac{4}{\gamma\varepsilon}\sum_{k=1}^K \frac{1}{\chi_k \overline{x}_k } + \frac{4K\eta}{\gamma\varepsilon} + \frac1\varepsilon.
\end{equation}
}
where the function $f$ is given in the statement of Lemma~\ref{lemma:lyapunov_function_bound}.
\end{theorem}

\begin{proof}
\final{
Note that we have ${(\chi_k + \varepsilon)\E\mleft[Z_{k,t+1}^\to\mright]} \leq \sqrt{(\chi_k + \varepsilon)/\overline{x}_k}\E\mleft[Z_{k,t+1}^\to\mright] \leq \E\mleft[ \widetilde{L}_{t+1} \mright]$. } Then the result follows from Lemma~\ref{lemma:lyapunov_function_bound} and Corollary~\ref{corollary:age_is_all_you_need}.
\end{proof}

We would also like to point out in the following remark that the age-based bandit learning policy can achieve a data freshness guarantee of $\mathbb{T}_{k,t} = O(\eta)$ for all $k$ and $t$, although it isn't specifically designed to optimize the TSLR $\mathbb{T}_{k,t}$. It's also worth noting that queue-length-based policies (e.g. \cite{li2019}) can also achieve similar guarantees (see Lemma 1 in \cite{wu2024}).

\begin{remark}[TSLR guarantees for the age-based policy]
\label{remark:tslr_guarantee}
Fix arm $k$ and define the \textit{``pseudo-time-since-last-reward"} $\widehat{\mathbb{T}}_{k,t}$ by the initial condition $\widehat{\mathbb{T}}_{k,1} = 0$ and update
\begin{equation}
\label{eq:time_since_last_departure}
\textstyle \widehat{\mathbb{T}}_{k,t+1} = \begin{cases}
\widehat{\mathbb{T}}_{k,t} + 1 & R_k(S_t,A_t) = 0 \land \uptau_k(\ell_{k,t}^\to) \leq t \\
0 & \text{otherwise}. \\
\end{cases}
\end{equation}
It can be shown by induction that $\widehat{\mathbb{T}}_{k,t} \leq Z_{k,t}^\to$ and $\mathbb{T}_{k,t}- \widehat{\mathbb{T}}_{k,t} \leq \Delta\uptau_k(\ell_{k,t}^\to-1)$ for all $t$. Therefore, we have the sample path bound $\mathbb{T}_{k,t} \leq {Z_{k,t}^\to} + \Delta\uptau_k(\ell_{k,t}^\to-1)$. Remark~\ref{remark:age_based_arrival_indepdendence} can be used to show that $\E\mleft[\Delta\uptau_k(\ell_{k,t}^\to-1)\mright] = \frac{1}{\chi_k + \varepsilon}$. Then taking expectations on our sample path bound gives that $\E\mleft[ \mathbb{T}_{k,t} \mright] \leq \E\mleft[ {Z_{k,t}^\to} \mright] + \frac{1}{\chi_k + \varepsilon}$, and finally, Lemma~\ref{lemma:lyapunov_function_bound} shows that $\E\mleft[ \mathbb{T}_{k,t} \mright] = O(\eta)$ for all $k$ and $t$.
\end{remark}

Finally, we derive the regret bound for the age-based bandit learning policy using the ``drift-plus-regret" technique introduced by \cite{li2019}.

\begin{theorem}
\label{theorem:regret}
The age-based learning policy  \eqref{eq:age_based_algorithm_def} suffers regret
\final{
\begin{equation}
\begin{aligned}
\mathrm{Reg}_T 
&\leq \frac{\varepsilon K T}{\gamma} + \frac{T}{\eta}\sum_{k=1}^K \frac{1}{\chi_k \overline{x}_k } + \underbrace{\sqrt{56 K I_{\max} T\log T} + \frac{K\pi^2}{3} }_{\text{from standard UCB regret analysis}} \\
\end{aligned}
\end{equation}
}
where $I_{\max} \triangleq \max_{a \in\mathcal{A}} \sum_{k=1}^K I_k(a)$.
\end{theorem}

\begin{hproof}
Define the \textit{$\varepsilon$-tight static policy} $\sigma^\varepsilon$ as a convex combination of the optimal static policy $\sigma^*$ and the $\gamma$-tight static policy $\sigma^\gamma$ as $\sigma^\varepsilon(a) \triangleq \left( 1 - \frac{\varepsilon}{\gamma} \right) \sigma^*(a) + \frac{\varepsilon}{\gamma} \sigma^\gamma(a)$ for all actions $a \in \mathcal{A}$. By the linearity of expectation, it can be shown that $\mathrm{Reg}^{\sigma_\varepsilon}_T \leq \frac{\varepsilon  K T}{\gamma}$. Then it remains to bound the \textit{``$\varepsilon$-tight regret"} $\varepsilon\text{-}\mathrm{Reg}_T \triangleq \mathrm{Reg}_T - \mathrm{Reg}^{\sigma_\varepsilon}_T$. Using Lemma~\ref{lemma:existence_of_static_policy}, Lemma~\ref{lemma:head_of_line_age},  Remark~\ref{remark:age_based_arrival_indepdendence}, and the definition of the Algorithm~\ref{alg:age_based}, after much work, we obtain the drift bound  \final{$\E\mleft[\Delta L_t \mright] \leq \eta \sum_{k=1}^K \E\mleft[ (I_k(A_t) - I_k(A_t^\varepsilon))U_{k,t} \mright] + \sum_{k=1}^K \frac{1}{\chi_k \overline{x}_k}$}. Next we apply the \textit{``drift-plus-regret"} technique: Dividing by $\eta$, summing over $t \in [T]$, adding $\varepsilon\text{-}\mathrm{Reg}_T$ to both sides and noticing that $\sum_{t=1}^T\E\mleft[\Delta L_t \mright] = \E\mleft[ L_{T+1}  \mright] \geq 0$ gives 
\final{
\begin{equation}
\begin{aligned}
\varepsilon\text{-}\mathrm{Reg}_T 
&\leq \textstyle \sum_{t=1}^T\sum_{k=1}^K \E\mleft[ R_{k}(S_t, A_t^\varepsilon) - I_k(A_t^\varepsilon)\, U_{k,t} \mright] \\
&\hspace{-0.2cm}+ \sum_{t=1}^T\sum_{k=1}^K \E\mleft[ I_k( A_t)\, U_{k,t} - R_{k}(S_t, A_t)\mright] + \frac{T}{\eta}\sum_{k=1}^K \frac{1}{\chi_k \overline{x}_k }.
\end{aligned}
\end{equation}
}
The remaining terms can be dealt with using standard UCB regret analysis. We omit the full proof due to space limitations. 
\end{hproof}


\section{Robustness of the Age-Based Bandit Learning Policy under Abruptly Changing Conditions}
\label{section:robustness_of_age}

In a real wireless network, the channel conditions may not be perfectly i.i.d. and may suffer from abrupt drops in quality due to weather, congestion, etc. Therefore, in this section, we investigate what happens when the channel quality of a single link abruptly degrades, rendering its short-term throughput constraint infeasible. 

\subsection{Experimental Setup}
\label{section:experimental_setup}

To investigate this phenomenon, we consider a simple network with $K = 2$ nodes (bandit arms) where only one node can be scheduled per timeslot.
For timeslots $t \in \{1, 2, \ldots , \lceil T/6 \rceil - 1 \}$, both links have channel rates $\E\mleft[ X_1(S_t) \mright] = \E\mleft[ X_2(S_t) \mright] = 0.9$. 
The 
\final{throughput requirements are}
$\chi_1 = 0.8$ and $\chi_2 = 0.1$, which are initially feasible. However, at timeslot $t = \lceil T/6 \rceil$, the channel quality for link $k=1$ suddenly drops to $\E\mleft[ X_1(S_t) \mright] = 0.5$ and remains that way for all timeslots $t \in \{ \lceil T/6 \rceil, \ldots , \lceil 2T/3 \rceil -1 \}$. This renders the short-term throughput constraint for arm $k=1$ infeasible! At timeslot $t = \lceil 2T/3 \rceil$, the channel quality for link $k=1$ finally improves back to $\E\mleft[ X_1(S_t) \mright] = 0.9$ and remains that way for all timeslots $t \in \{ \lceil 2T/3 \rceil, \ldots , T \}$.

\begin{figure}
  \centering
  \subfigure[Short-Term Throughput $\overline{R}_{k,t}\lbrack W \rbrack$]{%
    \includegraphics[width=0.5\columnwidth]{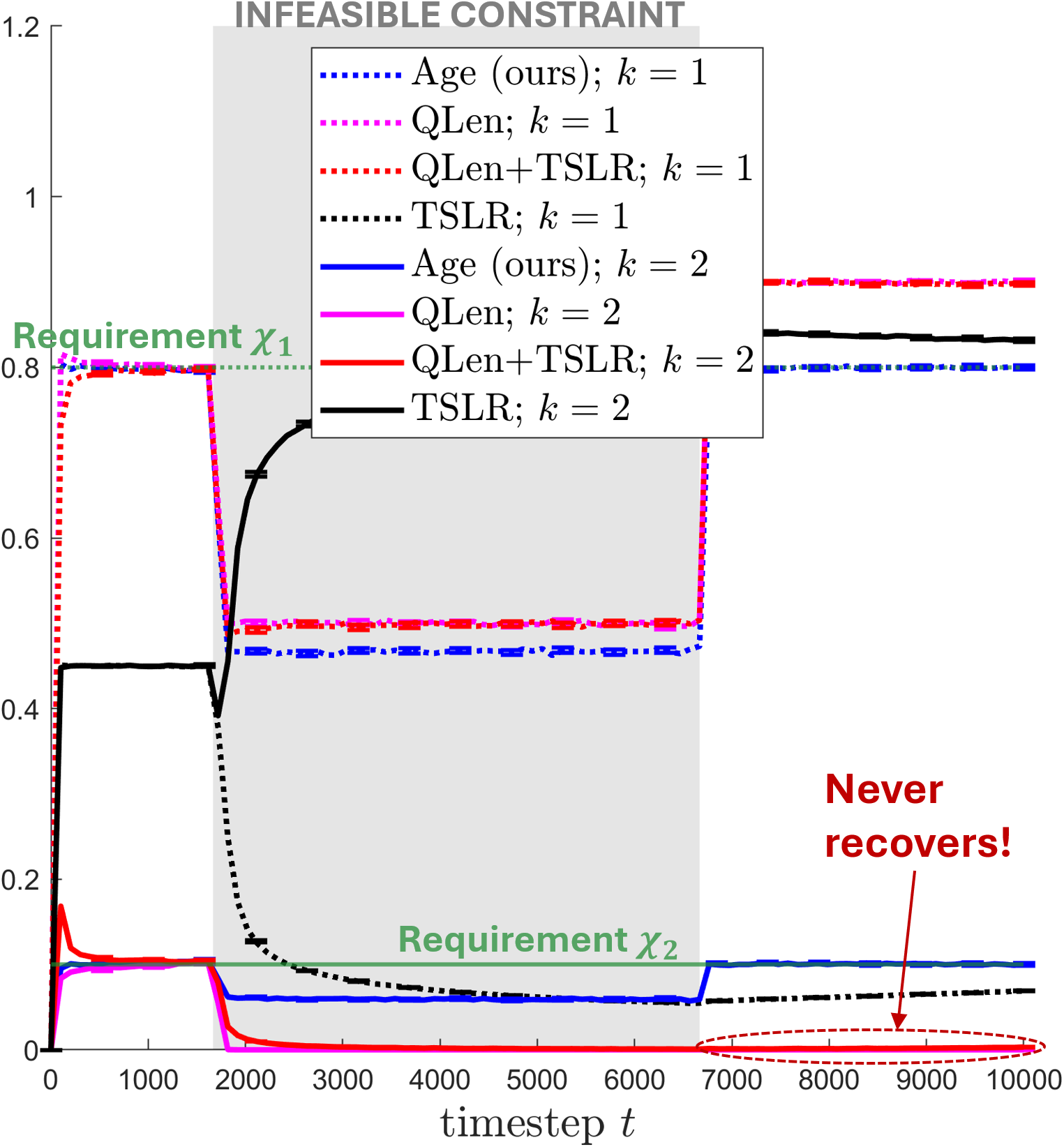}%
    \label{fig:infeasible_safety_rate}}%
  \subfigure[Time-Since-Last-Reward $\mathbb{T}_{k,t}$]{%
    \includegraphics[width=0.5\columnwidth]{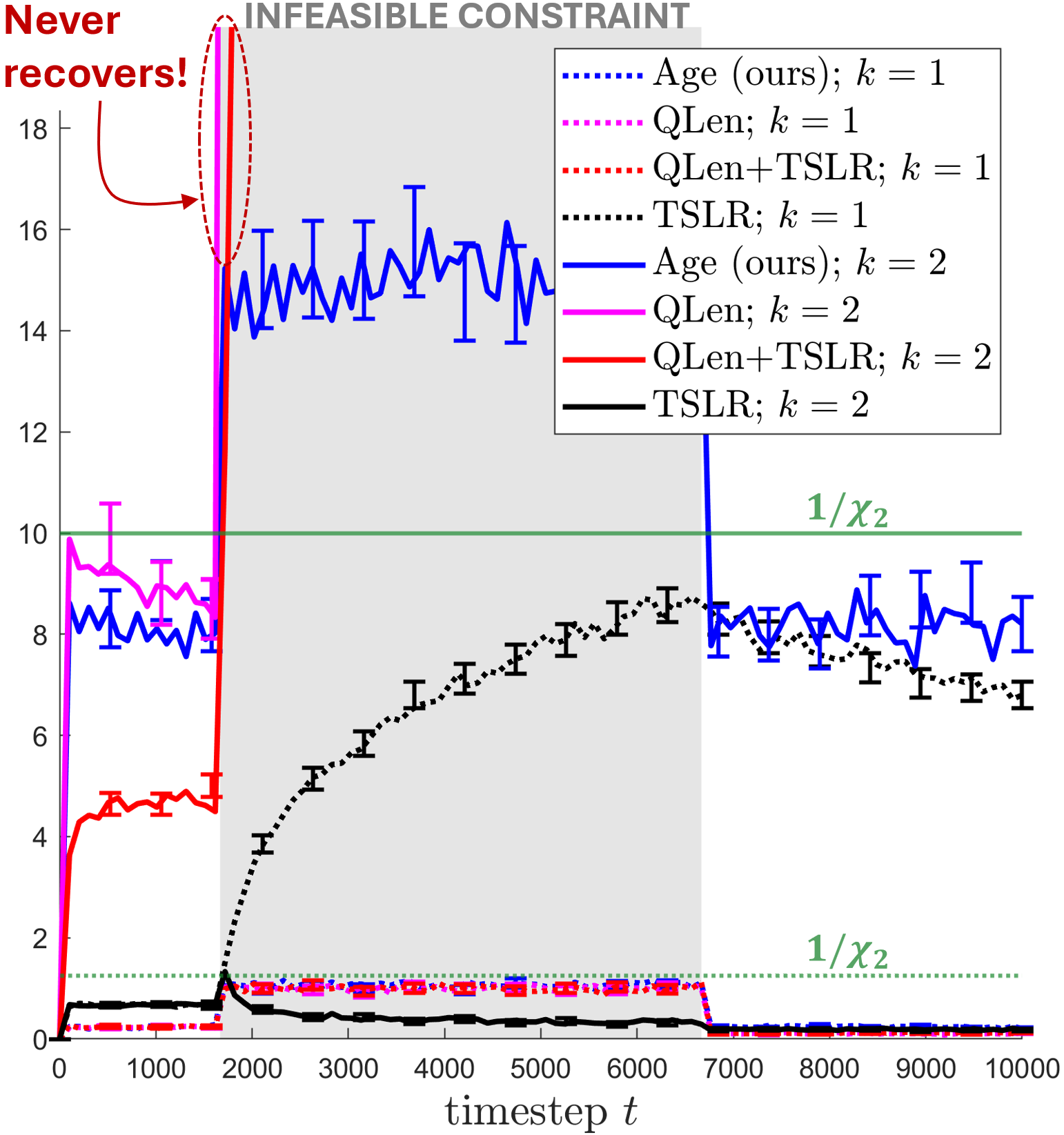}%
    \label{fig:infeasible_tss}}
  \caption{Comparison of policies for short-term throughput and freshness (TSLR) when the throughput constraint for arm $k=1$ becomes abruptly infeasible (shaded region). 1000 simulation runs are performed.}
  \label{fig:infeasible_constraints_overall}
\end{figure}

\begin{figure}
  \centering
  \subfigure[\final{Under queue-length-based policy}]{%
    \includegraphics[width=0.5\columnwidth]{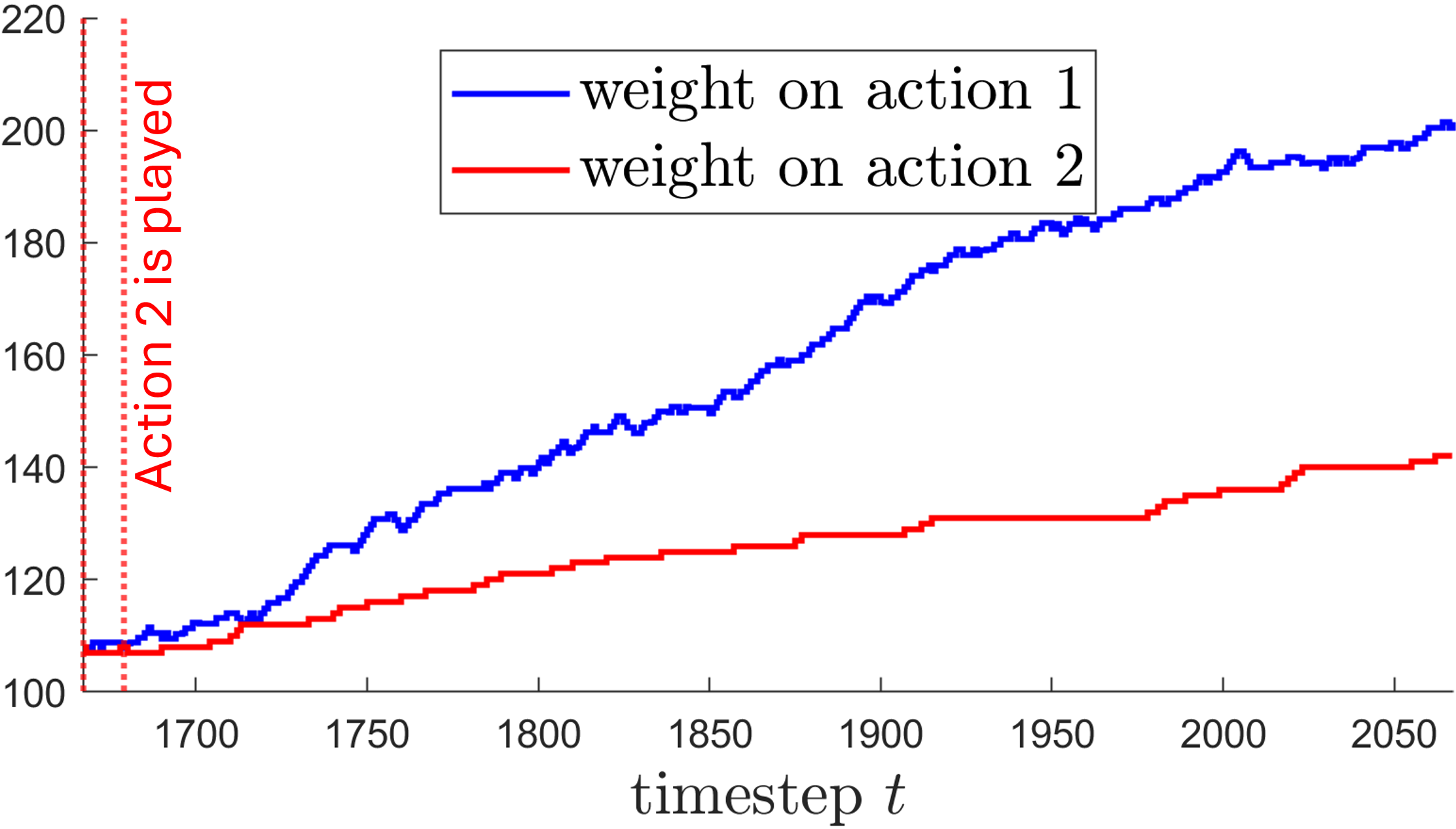}%
    \label{fig:length_based_sample}}%
  \subfigure[Under our age-based policy \eqref{eq:age_based_algorithm_def}]{%
    \includegraphics[width=0.5\columnwidth]{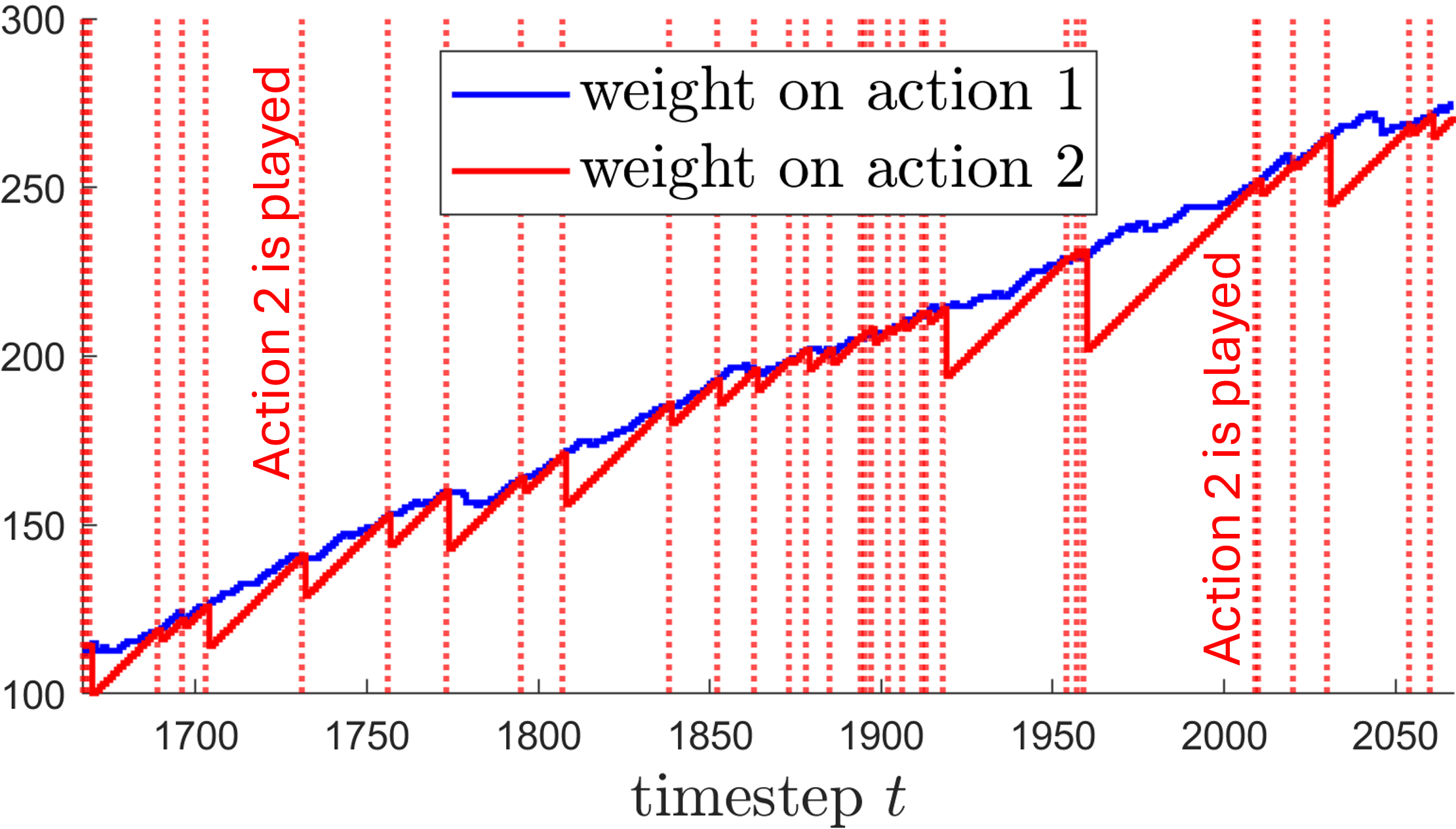}%
    \label{fig:age_based_sample}}
  \caption{Operation of the queue-length-based policy vs. our age-based policy \eqref{eq:age_based_algorithm_def} when the throughput constraint for arm $k=1$ is infeasible. Here, one simulation run from the results in Fig.~\ref{fig:infeasible_constraints_overall} is shown. \final{Vertical dotted red lines show when action 2 is played, and action 1 is played in all other timeslots.}}
  \label{fig:infeasible_constraints_sample_path}
\end{figure}

\subsection{Discussion of Results}

Fig.~\ref{fig:infeasible_constraints_overall} shows the result of these experiments and we use the same algorithm parameters from Section~\ref{section:empirical_performance} for consistency. In this figure, dotted lines are associated with arm $k=1$ and solid lines are associated with arm $k=2$. The key observation is that our age-based policy (shown in blue) is able to quickly recover from the period of temporary constraint infeasibility (the shaded region). Focusing first on the short-term throughput in Fig.~\ref{fig:infeasible_safety_rate}, the two queue-length-based policies (``QLen'' and ``QLen+TSLR'') neglect arm 2 and continue trying to schedule arm 1, whose queue length is blowing up. On the other hand, our age-based policy \eqref{eq:age_based_algorithm_def} never completely neglects arm 2. After the period of constraint infeasibility, the queue-length-based policies never recover. 
Interestingly, our age-based policy, which is not explicitly designed to target the TSLR metric, does much better than the QLen+TSLR policy on this metric, 
\final{
which never recovers from the period of constraint infeasibility.
}
Our intuition is that for QLen+TSLR, the queue length metric completely dominates any other metric considered in the algorithm. This also demonstrates the utility of the $O(\eta)$ TSLR control afforded by our age-based policy (see Remark~\ref{remark:tslr_guarantee}). As expected, the TSLR-based policy does well in the TSLR metric, but completely blunders the short-term throughput since it is not designed for that. For the TSLR-based policy, the opposite changes in short-term throughput during the period of constraint infeasibility are due to the changes in reward between the arms. 

To shed some light on why the head-of-line age does not suffer as much from periods of constraint infeasibility compared to the queue length, we plot the operation of the queue-length-based policy and our age-based policy for one simulation run in Fig.~\ref{fig:infeasible_constraints_sample_path}. Since we consider one action per arm, the ``weight" on an action $a = k$ is $\eta U_{k,t} + Z_{k,t}^\to$ under our age-based policy \eqref{eq:age_based_algorithm_def} and $\eta U_{k,t} + |\mathcal{Q}_{k,t}|$ under the queue-length-based policy. The key difference is in the dynamic of the head-of-line age $\Delta Z_{k,t}^\to$ given by Lemma~\ref{lemma:head_of_line_age} and the dynamic of the queue length $\Delta |\mathcal{Q}_{k,t}| = B_{k,t} - D_{k,t}$.  
\final{
In particular, the head-of-line age always increments by 1 when a reward is not obtained, and therefore the head-of-line age for the neglected arm always has a chance to ``catch up" to the other arm. On the other hand, the queue length grows much more slowly at a rate of $\chi_k + \varepsilon$. 
}
From these results \final{and the fact that the age-based and queue-length-based policies perform similarly in the i.i.d. setting (see Section~\ref{section:empirical_performance})}, we argue that this robustness under constraint infeasibility incentivizes the use of the head-of-line age in place of the queue length where appropriate in drift-plus-penalty style algorithms deployed to real-world systems. 

\subsection{Replication of the Experiment Using a Real Data Trace}

\begin{figure}
  \centering
  \subfigure[Short-Term Throughput $\overline{R}_{k,t}\lbrack W \rbrack$]{%
    \includegraphics[width=0.5\columnwidth]{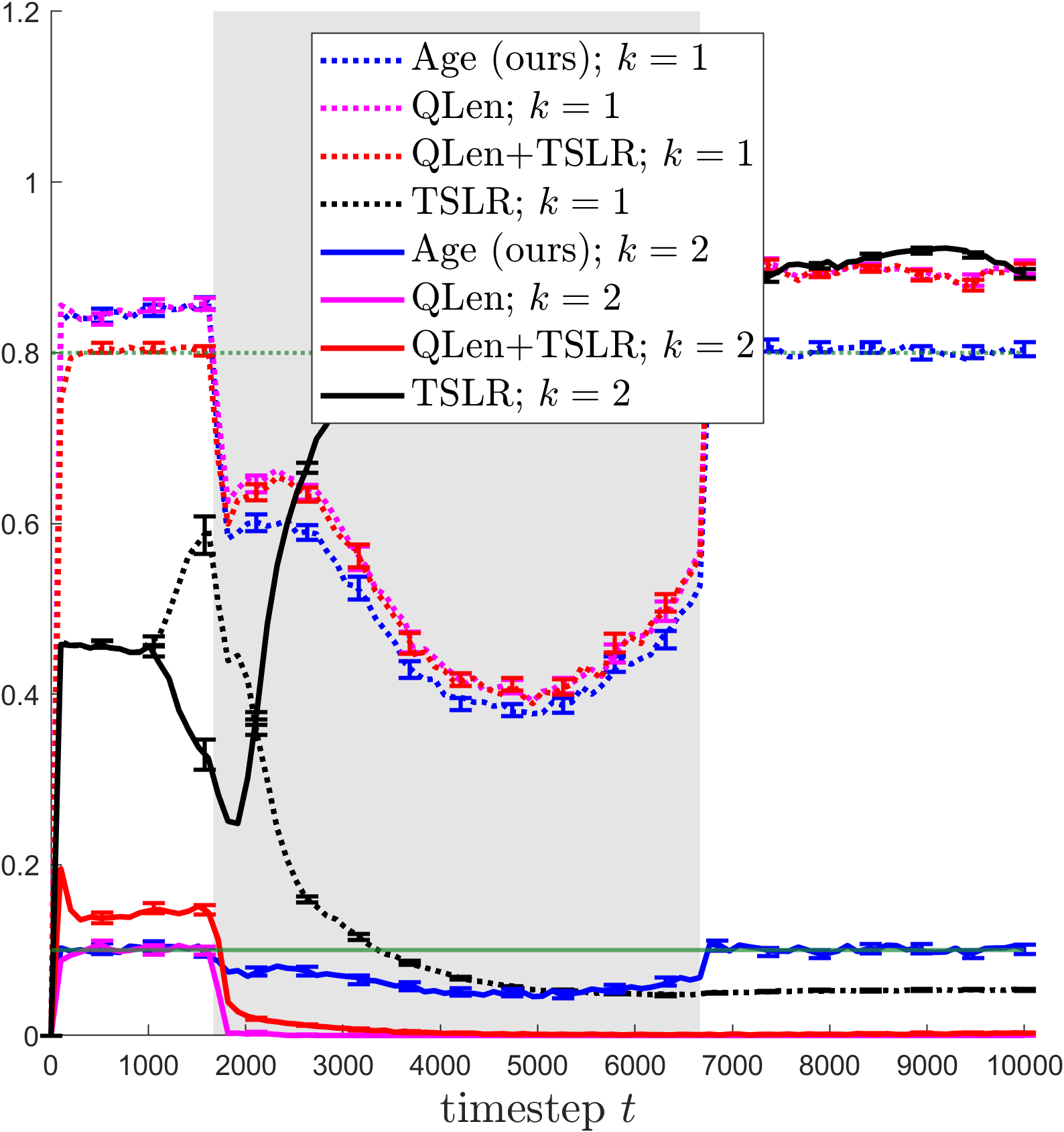}%
    \label{fig:infeasible_safety_rate_trace}}%
  \subfigure[Time-Since-Last-Reward $\mathbb{T}_{k,t}$]{%
    \includegraphics[width=0.5\columnwidth]{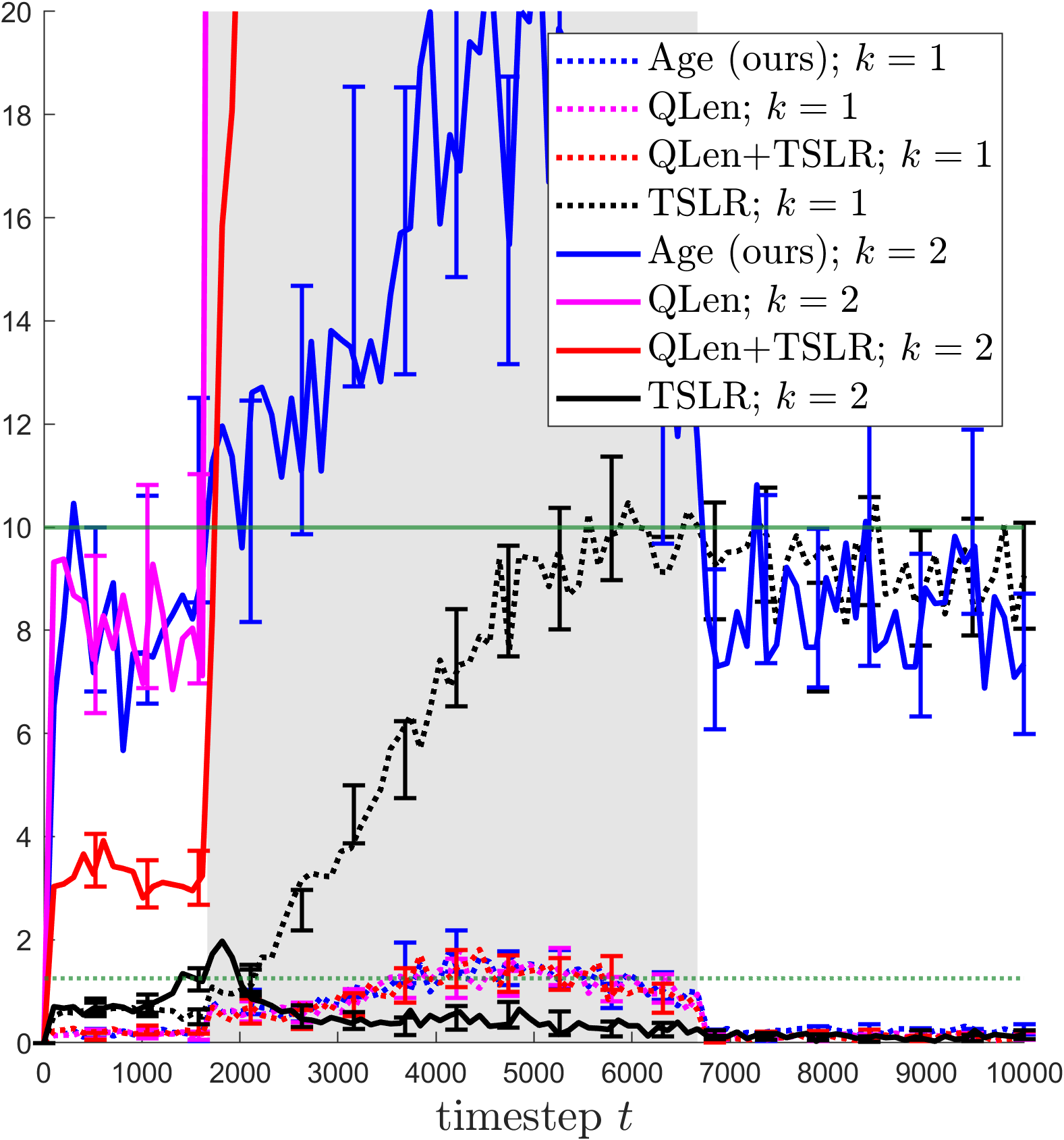}%
    \label{fig:infeasible_tss_trace}}
  \caption{Replication of the results in Fig.~\ref{fig:infeasible_constraints_overall} using a real data trace. 100 simulation runs are performed with Uniform(0,1000) random starting offsets.}
  \label{fig:infeasible_constraints_overall_trace}
\end{figure}

The authors in \cite{wu2023rateadaptation} provide an open source data trace of a wireless network with the SNR collected for $3 \times 10^4$ timeslots over 30 wireless channels. We set the bitrate at 1.46Gbps and select channels 12, 16, and 17 from their trace. The resulting channel rates are 0.5402, 0.9059, and 0.9012, which are similar to the ones we set in Section~\ref{section:experimental_setup}. Then we replicate the experiment from Section~\ref{section:experimental_setup} and show the results in Fig.~\ref{fig:infeasible_constraints_overall_trace} for completeness.

\section{Concluding Remarks and Future Work}
\label{section:conclusion}

In this paper, we investigated the robustness of using the head-of-line age instead of the virtual queue length to manage the throughput constraints in a learning-based wireless scheduling policy. We showed that the performance of such a policy can match the state-of-the-art both theoretically and empirically under i.i.d. network conditions. However, when the i.i.d. assumption is violated, leading to temporary constraint infeasibility, we demonstrated empirically that the age-based policy is significantly more stable and adaptive than the state-of-the-art policies. The next step, which is much more difficult, is to analytically quantify the robustness of the head-of-line age in the presence of temporarily infeasible constraints. We aim to investigate this in future work.

\bibliographystyle{IEEEtran}
\bibliography{refs}

\newpage

\include{appendix}

\end{document}

%% file: appendix.tex
\appendices

\section{Proof of Lemma~\ref{lemma:head_of_line_age}}

\begin{proof}
\final{Fix arm $k$ and timeslot $t$. According to the definition of age \eqref{eq:age_def}, we have for two consecutive delivery requests that their age difference satisfies $Z_{k,t}(\ell) -  Z_{k,t}(\ell+1) = \Delta \uptau_{k}(\ell)$. Also from the definition \eqref{eq:age_def}, we have 
$Z_{k,t+1}(\ell+1) = Z_{k,t}(\ell+1) + 1$. Combining these two observations gives that 
\begin{equation}
Z_{k,t+1}(\ell+1) = Z_{k,t}(\ell) - \Delta \uptau_{k}(\ell) + 1.
\label{eq:age_alert_time_drift}
\end{equation}
}
Note that according to the definition of age \eqref{eq:age_def}, we have $\left[Z_{k,t}(\ell_{k,t}^\to)\right]^- = \left[\uptau_k(\ell_{k,t}^\to) -  t \right]^+$. Then we can rewrite the definition of head-of-line age \eqref{eq:head_of_line_age}  as 
\begin{equation}
\label{eq:head_of_line_age_decomp}
Z_{k,t}^\to = Z_{k,t}(\ell_{k,t}^\to) + \left[\uptau_k(\ell_{k,t}^\to) -  t \right]^+.
\end{equation}
Now, consider two cases:

\textit{Case 1}: $D_{k,t} = 1$. Then $\ell_{k,t+1}^\to = \ell_{k,t}^\to + 1$, and therefore from \eqref{eq:age_alert_time_drift}, we have
\begin{equation}
Z_{k,t+1}(\ell_{k,t+1}^\to) = Z_{k,t+1}(\ell_{k,t}^\to + 1) = Z_{k,t}(\ell_{k,t}^\to) + 1 - \Delta \uptau_k(\ell_{k,t}^\to)
\end{equation}
From \eqref{eq:head_of_line_age_decomp}, we add $\left[\uptau_k(\ell_{k,t}^\to) -  t \right]^+ + \left[\uptau_k(\ell_{k,t+1}^\to) -  (t+1) \right]^+$ to both sides and rearrange to get
\begin{equation}
\label{eq:prelim_hol_age_drift}
\begin{aligned}
\Delta Z_{k,t}^\to &= 1 - \Delta \uptau_k(\ell_{k,t}^\to) \\
&\phantom{\,=\,}+  \left[\uptau_k(\ell_{k,t+1}^\to) - (t + 1) \right]^+ - \left[\uptau_k(\ell_{k,t}^\to) -  t \right]^+ \\
&\overset{(a)}{=} 1 - \Delta \uptau_k(\ell_{k,t}^\to) \\
&\phantom{\,=\,}+  \left[\uptau_k(\ell_{k,t}^\to)  - t + \Delta \uptau_k(\ell_{k,t}^\to) -1 \right]^+ - \left[\uptau_k(\ell_{k,t}^\to) -  t \right]^+ \\
&\overset{(b)}{=} \begin{cases}
\begin{aligned}
&1 - \Delta \uptau_k(\ell_{k,t}^\to) \\
&+ \left[\uptau_k(\ell_{k,t}^\to)  - t + \Delta \uptau_k(\ell_{k,t}^\to) -1 \right]^+ \\
\end{aligned} & \uptau_k(\ell_{k,t}^\to) \leq t \\
0 & \uptau_k(\ell_{k,t}^\to) > t
\end{cases} \\
\end{aligned}
\end{equation}
where $(a)$ is because $\ell_{k,t+1}^\to = \ell_{k,t}^\to + 1$ implies $\uptau_k(\ell_{k,t+1}^\to) = \uptau_k(\ell_{k,t}^\to) + \Delta \uptau_k(\ell_{k,t}^\to)$, and $(b)$ can be observed by noting that $\Delta \uptau_k(\ell_{k,t}^\to) \geq 1$. We can further decompose the case that $\uptau_k(\ell_{k,t}^\to) \leq t$ in the above as 
\begin{equation}
\begin{aligned}
& 1 - \Delta \uptau_k(\ell_{k,t}^\to) + \left[\uptau_k(\ell_{k,t}^\to)  - t + \Delta \uptau_k(\ell_{k,t}^\to) -1 \right]^+ \\
&= - \begin{cases}
t - \uptau_k(\ell_{k,t}^\to)  &  \Delta \uptau_k(\ell_{k,t}^\to) -1 > t - \uptau_k(\ell_{k,t}^\to) \\
\Delta \uptau_k(\ell_{k,t}^\to) - 1 & t - \uptau_k(\ell_{k,t}^\to) \geq \Delta \uptau_k(\ell_{k,t}^\to) -1 \\
\end{cases} \\
&= - \min\left\{ t - \uptau_k(\ell_{k,t}^\to), \Delta \uptau_k(\ell_{k,t}^\to) - 1 \right\}.
\end{aligned}
\end{equation}
Then we complete Case 1 by plugging the above back into \eqref{eq:prelim_hol_age_drift} to get
\begin{equation}
\begin{aligned}
\Delta Z_{k,t}^\to
&= \begin{cases}
-\min\left\{ t - \uptau_k(\ell_{k,t}^\to), \Delta \uptau_k(\ell_{k,t}^\to) - 1 \right\} &  \uptau_k(\ell_{k,t}^\to) \leq t \\
0 & \uptau_k(\ell_{k,t}^\to) > t.
\end{cases}
\end{aligned}
\end{equation}

\textit{Case 2}: $D_{k,t} = 0$. Then $\ell_{k,t+1}^\to = \ell_{k,t}^\to$ and therefore from the definition of age \eqref{eq:age_def}, we have
\begin{equation}
Z_{k,t+1}(\ell_{k,t+1}^\to) = Z_{k,t+1}(\ell_{k,t}^\to) = Z_{k,t}(\ell_{k,t}^\to) + 1.
\end{equation}
From \eqref{eq:head_of_line_age_decomp}, we add $\left[\uptau_k(\ell_{k,t}^\to) -  t \right]^+ + \left[\uptau_k(\ell_{k,t+1}^\to) -  (t+1) \right]^+$ to both sides and rearrange to get
\begin{equation}
\begin{aligned}
\Delta Z_{k,t}^\to 
&= 1 + \left[\uptau_k(\ell_{k,t+1}^\to) - (t + 1) \right]^+ - \left[\uptau_k(\ell_{k,t}^\to) -  t \right]^+ \\
&= 1 + \left[\uptau_k(\ell_{k,t}^\to) - t - 1 \right]^+ - \left[\uptau_k(\ell_{k,t}^\to) -  t \right]^+ \\
&= \begin{cases}
1  & \uptau_k(\ell_{k,t}^\to) \leq t \\
0 & \uptau_k(\ell_{k,t}^\to) > t .
\end{cases}
\end{aligned}
\end{equation}

Combining Case 1 and Case 2 gives that when $\uptau_k(\ell_{k,t}^\to) \leq t$, we have
\begin{equation}
\begin{aligned}
&\Delta Z_{k,t}^\to \\
&= - D_{k,t}\min\left\{ t - \uptau_k(\ell_{k,t}^\to), \Delta \uptau_k(\ell_{k,t}^\to) - 1 \right\} + (1-D_{k,t}) \\
&= 1 - D_{k,t}\left(1 + \min\left\{ t - \uptau_k(\ell_{k,t}^\to), \Delta \uptau_k(\ell_{k,t}^\to) - 1 \right\}\right) \\
&= 1 - D_{k,t} \min\left\{ t - \uptau_k(\ell_{k,t}^\to) + 1, \Delta \uptau_k(\ell_{k,t}^\to) \right\} \\
\end{aligned}
\end{equation}
and $\Delta Z_{k,t}^\to = 0$ when $\uptau_k(\ell_{k,t}^\to) > t$.
\end{proof}

\section{Proof of Lemma~\ref{lemma:drift_lemma}}

\begin{proof}
Fix $t$. We obtain an upper bound on $\E\mleft[ e^{\zeta Y_{t+1}} \mright]$ by considering two cases. 

\textit{Case 1}: $Y_t \geq \varphi$. Note that $e^x \leq 1 + x + \frac{x^2}{2}e^{|x|} \,\,\forall x$. Then 
\begin{equation}
\label{eq:case1_mgf_bound}
\begin{aligned}
\E\mleft[ e^{\zeta \Delta Y_t} \mid \mathcal{F}_t \mright]
&\leq 1 + \zeta \E\mleft[  \Delta Y_t \mid \mathcal{F}_t \mright] + \frac{\zeta^2}{2}\E\mleft[  |\Delta Y_t|^2 e^{\zeta |\Delta Y_t|} \mid \mathcal{F}_t \mright] \\
&\overset{(a)}{\leq} 1 -\zeta\rho + \frac{\zeta^2}{2}\E\mleft[ |\Delta Y_t|^2 e^{\zeta |\Delta Y_t|} \mid \mathcal{F}_t \mright] \\
&\overset{(b)}{\leq} 1 -\zeta\rho + \frac{\zeta^2}{2}\E\mleft[ |\Delta Y_t|^2 e^{\frac{\theta}{2} |\Delta Y_t|} \mid \mathcal{F}_t \mright] \\
\end{aligned}
\end{equation}
where $(a)$ uses drift condition (i) and $(b)$ is because  $\zeta \leq \frac{\rho \theta^2}{8 M} $ and $\frac{\rho \theta^2}{8 M} \leq \frac{\theta}{2}$ since $\rho,\theta \in (0, 1]$ and $M \geq 1$. Next, we focus on the conditional expectation on the right-hand side in the above. Note that $x^2 \leq \frac{2}{\alpha^2}e^{\alpha x}$ for all $x \geq 0$ and $\alpha > 0$. Then setting $\alpha = \theta/2$ gives
\begin{equation}
|\Delta Y_t|^2 \leq \frac{8}{\theta^2}e^{\frac{\theta}{2} |\Delta Y_t|} \implies |\Delta Y_t|^2 e^{\frac{\theta}{2} |\Delta Y_t|} \leq \frac{8}{\theta^2} e^{\theta |\Delta Y_t|}.
\end{equation}
Using the above along with drift condition (ii) to further bound \eqref{eq:case1_mgf_bound} gives
\begin{equation}
\begin{aligned}
\E\mleft[ e^{\zeta \Delta Y_t} \mid \mathcal{F}_t \mright]
&\leq  1 -\zeta\rho + \frac{4 \zeta^2 M}{\theta^2} \overset{(a)}{\leq}  1 -\frac{\zeta\rho}{2}
\end{aligned}
\end{equation}
where $(a)$ is because $\zeta \leq \frac{\rho \theta^2}{8 M}$. Then we complete Case 1 by recalling that $Y_t$ is $\mathcal{F}_t$-measurable and writing 
\begin{equation}
\E\mleft[ e^{\zeta  Y_{t+1}} \mid \mathcal{F}_t \mright] = e^{\zeta Y_t}\E\mleft[ e^{\zeta \Delta Y_t} \mid \mathcal{F}_t \mright] \leq e^{\zeta Y_t}\left(1 -\frac{\zeta\rho}{2}\right)
\end{equation}

\textit{Case 2}: $Y_t < \varphi$. Then $Y_{t+1} \leq \varphi + |\Delta Y_t|$ and therefore
\begin{equation}
\begin{aligned}
\E\mleft[ e^{\zeta  Y_{t+1}} \mid \mathcal{F}_t \mright] 
&\leq e^{\zeta \varphi}\E\mleft[ (e^{\theta |\Delta Y_t|})^{{\zeta}/{\theta}} \mid \mathcal{F}_t \mright] \\ &\overset{(a)}{\leq} e^{\zeta \varphi}\left(\E\mleft[ e^{\theta |\Delta Y_t|} \mid \mathcal{F}_t \mright]\right)^{\zeta/\theta}
&\overset{(b)}{\leq} e^{\zeta \varphi} M^{\zeta/\theta}
\end{aligned}
\end{equation}
where $(a)$ is by Jensen's inequality since $\zeta/\theta \in (0, 1]$ and $(b)$ applies drift condition (ii).

Combining Case 1 and Case 2 and taking expectations gives
\begin{equation}
\begin{aligned}
\underbrace{\E\mleft[ e^{\zeta  Y_{t+1}}\mright]}_{x_{t+1}}
&\leq \big(1 -\underbrace{{\zeta\rho}/{2}}_{a}\big)\underbrace{\E\mleft[ e^{\zeta  Y_{t}}\mright]}_{x_t} + \underbrace{e^{\zeta \varphi} M^{\zeta/\theta}}_{b}.
\end{aligned}
\end{equation}
So, we have the form $x_{t+1} \leq (1-a) x_t + b$ for all $t > 1$ where $a \in (0,1)$ and $b > 0$. Then it's easy to show by induction that $x_t \leq {(1-a)^{t-1}x_1 + b / a}$ for all $t \geq 1$.  

\end{proof}

\section{Proof of Lemma~\ref{lemma:lyapunov_function_bound}}
\final{
Before presenting the proof, the following remark deals with the interarrival times, which will be frequently used. 
\begin{remark}[history independence of interarrival times]
\label{remark:interarrival_times}
Note that for any measurable function $\psi : \mathbb{Z}_{+}^K \to \mathbb{R}$ and timeslot $t$,
\begin{equation}
\begin{aligned}
&\E\mleft[ \psi\mleft( (\Delta\uptau_k(\ell_{k,t}^\to))_{k=1}^K \mright) \mid \mathcal{H}_t^\to \mright] \\
&= \sum_{(\ell_k)_{k=1}^K} \E\mleft[ \1\mleft\{ (\ell_{k,t}^\to)_{k=1}^K = (\ell_k)_{k=1}^K \mright\} \psi\mleft( (\Delta\uptau_k(\ell_k))_{k=1}^K \mright) \mid \mathcal{H}_t^\to \mright] \\
&\overset{(a)}{=} \sum_{(\ell_k)_{k=1}^K} \1\mleft\{ (\ell_{k,t}^\to)_{k=1}^K = (\ell_k)_k \mright\} \E\mleft[  \psi\mleft( (\Delta\uptau_k(\ell_k))_{k=1}^K \mright) \mid \mathcal{H}_t^\to \mright] \\
&\overset{(b)}{=} \sum_{(\ell_k)_{k=1}^K} \1\mleft\{ (\ell_{k,t}^\to)_{k=1}^K = (\ell_k)_k \mright\} \E\mleft[  \psi\mleft( (\Delta\uptau_k(\ell_k))_{k=1}^K \mright) \mright] \\
&\overset{(c)}{=} \E\mleft[  \psi\mleft( (\Delta\uptau_k(1))_{k=1}^K \mright) \mright] \underbrace{\textstyle\sum_{(\ell_k)_{k=1}^K} \1\mleft\{ (\ell_{k,t}^\to)_{k=1}^K = (\ell_k)_k \mright\}}_{=\,1}
\end{aligned}
\end{equation}
where $(a)$ is because each $\ell_{k,t}^\to$ is the number of departures before $t$ plus one (see \eqref{eq:next_to_depart}), which is $\mathcal{H}_t^\to$-measurable, $(b)$ is because for all $k$, $\mathcal{H}_t^\to$ is independent from $\Delta \uptau_k(\ell)$ for all $\ell \geq \ell_{k,t}^\to$ (see Remark~\ref{remark:age_based_arrival_indepdendence}), and $(c)$ is because each $\Delta \uptau_k(\ell)$ is an i.i.d. process.
\end{remark}
We are now ready to present the proof of Lemma~\ref{lemma:lyapunov_function_bound}.
}
\begin{proof}
We proceed from Lemma~\ref{lemma:elementary_drift_bound}. We first manipulate the conditional expectation of $\mathtt{t2}$ as
\final{
\begin{equation}
\begin{aligned}
\label{eq:t2_conditional_expectation_lower_bound}
&\E\mleft[ \mathtt{t2} \mid \mathcal{H}_t^\to \mright] \\
&\overset{(a)}{=} \textstyle\sum_{k=1}^K \frac{(\chi_k + \varepsilon)}{\overline{x}_k} I_k(A_t) Z_{k,t}^\to \E\mleft[ X_k(S_t) \Delta \uptau_k(\ell_{k,t}^\to)  \mid \mathcal{H}_t^\to \mright] \\
&\overset{(b)}{=} \textstyle\sum_{k=1}^K (\chi_k + \varepsilon)  I_k(A_t) Z_{k,t}^\to \E\mleft[  \Delta \uptau_k(\ell_{k,t}^\to) \mid \mathcal{H}_t^\to \mright] \\
&\overset{(c)}{=} \sum_{k=1}^K I_k(A_t) Z_{k,t}^\to \\
\end{aligned}
\end{equation}
where $(a)$ is because $A_t$ and $Z_{k,t}^\to$ are $\mathcal{H}_t^\to$-measurable under an age-based policy, $(b)$ is because $S_t$ is independent from $\Delta \uptau_k(\ell_{k,t}^\to)$ and $\mathcal{H}_t^\to$, and $(c)$ if from Remark~\ref{remark:interarrival_times} and the fact that the interarrival times for each arm $k$ are $\text{Geometric}(\chi_k + \varepsilon)$-distributed. 

}
\final{
We use this to lower bound 
\begin{equation}
\label{eq:ucb_plus_t2_conditional_expectation_lower_bound}
\begin{aligned}
&\textstyle \E\mleft[ \eta \sum_{k=1}^K I_k(A_t) U_{k,t}  \mid \mathcal{H}_t^\to \mright] + \E\mleft[ \mathtt{t2} \mid \mathcal{H}_t^\to \mright] \\
&\overset{(a)}{=}\textstyle \E\mleft[  \sum_{k=1}^K I_k(A_t) \left(\eta U_{k,t} + Z_{k,t}^\to \right)  \mid \mathcal{H}_t^\to \mright] \\
&\overset{(b)}{\geq}\textstyle \E\mleft[  \sum_{k=1}^K I_k(A_t^\gamma) \left(\eta U_{k,t} + Z_{k,t}^\to \right)  \mid \mathcal{H}_t^\to \mright] \\
&\geq \textstyle \sum_{k=1}^K  Z_{k,t}^\to\E\mleft[ I_k(A_t^\gamma)  \mid \mathcal{H}_t^\to \mright]  \\
&\overset{(c)}{\geq} \textstyle  \sum_{k=1}^K  Z_{k,t}^\to\E\mleft[ I_k(A_t^\gamma)  \mright]
\overset{(d)}{\geq}  \textstyle \sum_{k=1}^K  Z_{k,t}^\to \frac{(\chi_k + \gamma)}{\overline{x}_k}
\end{aligned}
\end{equation}
where $(a)$ is from \eqref{eq:t2_conditional_expectation_lower_bound} and the fact that $A_t$ and $Z_{k,t}^\to$ are $\mathcal{H}_t^\to$-measurable, $(b)$ is by definition \eqref{eq:age_based_algorithm_def} of the age-based learning algorithm, $(c)$ is because $A_t^\gamma$ and $\mathcal{H}_t^\to$ are independent, and $(d)$ is from \eqref{eq:gamma_tight_static_policy}.
}
\final{
Then we have
\begin{equation}
\label{eq:t1_minus_t2_conditional_expectation_upper_bound}
\begin{aligned}
&\E\mleft[ \mathtt{t1} - \mathtt{t2} \mid \mathcal{H}_t^\to \mright] \\
&\leq \E\mleft[ \mathtt{t1} - \mathtt{t2} + \eta K - \eta \sum_{k=1}^K I_k(A_t) U_{k,t} \mid \mathcal{H}_t^\to \mright] \\
&\overset{(a)}{\leq} \E\mleft[ \sum_{k=1}^K \frac{(\chi_k + \varepsilon)}{\overline{x}_k} Z_{k,t}^\to - \sum_{k=1}^K Z_{k,t}^\to \frac{(\chi_k + \gamma)}{\overline{x}_k}  \mid \mathcal{H}_t^\to \mright]  + \eta K \\
&\overset{(b)}{=} (\varepsilon - \gamma) \sum_{k=1}^K \frac{Z_{k,t}^\to}{\overline{x}_k}  + \eta K 
\end{aligned}
\end{equation}
where $(a)$ is due to \eqref{eq:ucb_plus_t2_conditional_expectation_lower_bound} and the definition of $\mathtt{t1}$, and $(b)$ uses the fact that each $Z_{k,t}^\to$ is $\mathcal{H}_t^\to$-measurable.
}
We upper bound the conditional expectation of $\mathtt{t3}$ by
\final{
\begin{equation}
\label{eq:t3_conditional_expectation_upper_bound}
\begin{aligned}
&\E\mleft[ \mathtt{t3} \mid \mathcal{H}_t^\to \mright] 
\overset{(a)}{=} \sum_{k=1}^K \frac{(\chi_k + \varepsilon)}{2\, \overline{x}_k}\left( 1 + \E\mleft[\left[\Delta \uptau_k(1)\right]^2 \mright] \right) \\
&\overset{(b)}{=} \sum_{k=1}^K \frac{(\chi_k + \varepsilon)}{2\, \overline{x}_k}\left( 1 + \frac{2-(\chi_k + \varepsilon)}{(\chi_k + \varepsilon)^2} \right)  \overset{(c)}{\leq} \sum_{k=1}^K \frac{1}{\overline{x}_k\chi_k}
\end{aligned}
\end{equation}
where $(a)$ is from Remark~\ref{remark:interarrival_times},  $(b)$ is because the second moment of a Geometric($\chi_k + \varepsilon$) random variable is $\frac{2-(\chi_k + \varepsilon)}{(\chi_k + \varepsilon)^2}$, and $(c)$ results from the manipulations $\frac{(\chi_k + \varepsilon)}{2\, \overline{x}_k}\left( 1 + \frac{2-(\chi_k + \varepsilon)}{(\chi_k + \varepsilon)^2} \right) = \frac{\chi_k + \varepsilon}{2 \overline{x}_k} + \frac{2-(\chi_k+\varepsilon)}{2\overline{x}_k(\chi_k+\varepsilon)} \leq \frac{\chi_k + \varepsilon}{2 \overline{x}_k(\chi_k+\varepsilon)} + \frac{2-(\chi_k+\varepsilon)}{2\overline{x}_k(\chi_k+\varepsilon)} = \frac{1}{\overline{x}_k(\chi_k+\varepsilon)} \leq \frac{1}{\overline{x}_k\chi_k}$. 
}
Define $\rho \triangleq \frac{\gamma - \varepsilon}{2}$ and 
\final{
$\varphi \triangleq \frac{1}{\rho}\left(\eta K + \sum_{k=1}^K \frac{1}{\overline{x}_k \chi_k }\right)$
}
and note that $\rho, \varphi > 0$. Then from Lemma~\ref{lemma:elementary_drift_bound}, \eqref{eq:t1_minus_t2_conditional_expectation_upper_bound}, and \eqref{eq:t3_conditional_expectation_upper_bound}, 
\final{
\begin{equation}
\E\mleft[ \Delta L_t \mid \mathcal{H}_t^\to \mright] \leq -2\rho \sum_{k=1}^K \frac{Z_{k,t}^\to}{\overline{x}_k} + \rho \varphi \overset{(a)}{\leq} -2\rho \widetilde{L}_t + \rho \varphi.
\end{equation}
where $(a)$ is because $\sqrt{\chi_k + \varepsilon}  \leq 1$ and $\frac{1}{\sqrt{\overline{x}_k}} \leq \frac{1}{\overline{x}_k}$ and so $\widetilde{L}_t = \|(\sqrt{(\chi_k + \varepsilon)/\overline{x}_k}Z_{k,t}^\to)_{k=1}^K\|_2 \leq \|(Z_{k,t}^\to / \overline{x}_k)_{k=1}^K \|_2 \leq \sum_{k=1}^K \frac{Z_{k,t}^\to}{\overline{x}_k}$.
}
Then from \eqref{eq:lyapunov_function_connection}, we have the first drift condition:
\begin{equation}
\widetilde{L}_t > \varphi \implies \E\mleft[ \Delta \widetilde{L}_t \mid \mathcal{H}_t^\to \mright] \leq -2\rho + \frac{\rho\varphi}{\widetilde{L}_t} \leq -\rho.
\end{equation}

Next we need to derive the second drift condition. We bound
\final{
\begin{equation}
\begin{aligned}
| \Delta \widetilde{L}_t | 
&\overset{(a)}{\leq} \big\| \Delta \big(\textstyle\sqrt{\frac{\chi_k + \varepsilon}{\overline{x}_k}}Z_{k,t}^\to \big)_{k=1}^K \big\|_2 = \big\|  \big(\textstyle\sqrt{\frac{\chi_k + \varepsilon}{\overline{x}_k}}\Delta Z_{k,t}^\to \big)_{k=1}^K \big\|_2  \\
&\leq \sum_{k=1}^K \frac{1}{\overline{x}_k} |\Delta Z_{k,t}^\to| 
\overset{(b)}{\leq} \sum_{k=1}^K \frac{2}{\overline{x}_k}\Delta \uptau_k(\ell_{k,t}^\to)
\end{aligned}
\end{equation}
}
where $(a)$ is by the reverse triangle inequality and $(b)$ is because ${|1- x|} \leq 1 + x$ for all $x \geq 0$, and so from Lemma~\ref{lemma:head_of_line_age}, we have that $|\Delta Z_{k,t}^\to| \leq 1 + \Delta \uptau_k(\ell_{k,t}^\to) \leq 2\Delta \uptau_k(\ell_{k,t}^\to)$. Define 
\final{$\theta_{\max} \triangleq -\frac{\overline{x}_{\min}}{2} \log(1-\chi_{\min})$} 
and consider \final{$\theta \in (0,\theta_{\max}]$}.
Then from the above, we have 
\final{
\begin{equation}
\begin{aligned}
&\E\mleft[ e^{\theta \left| \Delta \widetilde{L}_t \right|} \mid \mathcal{H}_t^\to \mright] 
\leq \E\mleft[ \prod_{k=1}^K \exp\mleft({\theta \frac{2}{\overline{x}_k}\Delta \uptau_k(\ell_{k,t}^\to)}\mright) \mid \mathcal{H}_t^\to \mright] \\
&\overset{(a)}{=}  \E\mleft[ \prod_{k=1}^K \exp\mleft({\theta \frac{2}{\overline{x}_k} \Delta \uptau_k(1)}\mright)  \mright] 
\overset{(b)}{=}  \prod_{k=1}^K \E\mleft[ \exp\mleft({\theta \frac{2}{\overline{x}_k}\Delta \uptau_k(1)}\mright)  \mright]  \\
&\overset{(c)}{=}  \prod_{k=1}^K \frac{(\chi_k+\varepsilon)e^{2\theta/\overline{x}_k}}{1 - (1-(\chi_k+\varepsilon))e^{2\theta/\overline{x}_k}}  \\
&\leq  \prod_{k=1}^K \frac{\chi_k\, e^{2\theta/\overline{x}_k}}{1 - (1-\chi_k)e^{2\theta/\overline{x}_k}} \triangleq M
\end{aligned}
\end{equation}
}
where $(a)$ \final{is from Remark~\ref{remark:interarrival_times}},  $(b)$ is because the interarrival times are i.i.d, \final{
and plugging in the MGF of a ${\text{Geometric}(\chi_k + \varepsilon)}$ random variable in $(c)$ is well defined since $\frac{2\theta}{\overline{x}_k} \leq \frac{2\theta_{\max}}{\overline{x}_k} = -\frac{\overline{x}_{\min}}{\overline{x}_{k}}\log(1-\chi_{\min}) \leq -\log(1-\chi_{\min}) < -\log(1-(\chi_{k}+\varepsilon))$.
}
This establishes the second drift condition in Lemma~\ref{lemma:drift_lemma}. 

Set $\zeta \triangleq \frac{\rho\theta^2}{8 M}$. Then using the two drift conditions with Lemma~\ref{lemma:drift_lemma} and applying Jensen's inequality gives 
\final{
\begin{equation}
\begin{aligned}
&\E\mleft[ \widetilde{L}_t \mright] 
\overset{(a)}{\leq} \frac{1}{\zeta}\log\mleft(\E\mleft[ e^{\zeta\widetilde{L}_t} \mright]\mright) \overset{(b)}{\leq} \frac{1}{\zeta}\log\mleft(\frac{4}{\zeta\rho}e^{\zeta \varphi} M^{\zeta/\theta}\mright) \\
&= \frac{1}{\zeta}\log\mleft(\frac{4}{\zeta\rho}\mright) + \varphi + \frac{1}{\theta}\log M \\
&= \frac{8M}{\rho\theta^2}\left(2\log\mleft(\frac{1}{\rho}\mright) + \log\mleft(\frac{32M}{\theta^2}\mright) \right) + \frac{1}{\theta}\log M + \varphi \\
&\overset{(c)}{\leq} \frac{2\log(1/\rho)}{\rho} \left( \frac{9M}{\theta^2}\log\mleft(\frac{32M}{\theta^2}\mright) \right) + \varphi \\
&\overset{(d)}{\leq} \frac{8\log(4/\gamma)}{\gamma} \left( \frac{9M}{\theta^2}\log\mleft(\frac{32M}{\theta^2}\mright) \right) + \frac{4}{\gamma}\left(\eta K + \sum_{k=1}^K \frac{1}{\chi_k \overline{x}_k}\right)
\end{aligned}
\end{equation}
}
where $(a)$ is due to Jensen's inequality, $(b)$ applies Lemma~\ref{lemma:drift_lemma} using the two established drift conditions (note that the first term in the lemma is $\leq 1$ since $ \widetilde{L}_t = 0$), $(c)$ is because $2\log(2/x) > 1$ for $x\in (0,1)$ and $\rho = \frac{\gamma - \varepsilon}{2}$, and $(d)$ is because since $\varepsilon \leq \gamma / 2$, we have $\rho = (\gamma - \varepsilon)/2 \geq \gamma / 4$. Optimizing over $\theta \in (0,\theta_{\max})$ gives the result. 
\end{proof}

The following lemma derives a sample path bound on the drift of the weighted quadratic Lyapunov function.
\begin{lemma}
\label{lemma:elementary_drift_bound}
For all 
\final{arms $k$}
and timeslots $t$, we have $\Delta L_t \leq \mathtt{t1} -  \mathtt{t2} +  \mathtt{t3}$
where 
\final{
\begin{equation}
\begin{aligned}
\mathtt{t1} &\triangleq \textstyle\sum_{k=1}^K \frac{(\chi_k + \varepsilon)}{\overline{x}_k} Z_{k,t}^\to ,\\
\mathtt{t2} &\triangleq \textstyle\sum_{k=1}^K \frac{(\chi_k + \varepsilon)}{\overline{x}_k} R_{k}(S_t, A_t) \Delta \uptau_k(\ell_{k,t}^\to) Z_{k,t}^\to ,\\
\mathtt{t3} &\triangleq   \textstyle\sum_{k=1}^K \frac{(\chi_k + \varepsilon)}{2\overline{x}_k}\left( 1 + \left[\Delta \uptau_k(\ell_{k,t}^\to)\right]^2 \right). \\
\end{aligned}
\end{equation}
}
\end{lemma}

\begin{proof}
We first need to show that for all 
\final{arms $k$}
and timeslots $t$, we have
\begin{equation}
\label{eq:hola_next_step_upper_bound}
Z_{k,t+1}^\to \leq \left[ Z_{k,t}^\to + 1 - D_{k,t}\Delta \uptau_k(\ell_{k,t}^\to) \right]^+.
\end{equation}
\final{Fix arm $k$ and timeslot $t$}.
We consider three cases:

\textit{Case 1}: $\uptau_k(\ell_{k,t}^\to) \geq t$. Then $Z_{k,t}^\to = 0$ \final{by definition (see \eqref{eq:age_def} and \eqref{eq:head_of_line_age})} and so Lemma~\ref{lemma:head_of_line_age} gives
\begin{equation}
\begin{aligned}
Z_{k,t+1}^\to 
&= \1\{ \uptau_k(\ell_{k,t}^\to) = t \}(1-D_{k,t})(Z_{k,t}^\to + 1) \\ 
&\leq (1-D_{k,t})(Z_{k,t}^\to + 1).
\end{aligned}
\end{equation}

\textit{Case 2}: $\uptau_k(\ell_{k,t}^\to) < t < \uptau_k(\ell_{k,t+1}^\to)$. Then $\ell_{k,t}^\to$ is the only delivery request in the queue at the beginning of $t$ and it has age $Z_{k,t}^\to = t - \uptau_k(\ell_{k,t}^\to)$. Then from Lemma~\ref{lemma:head_of_line_age}, we have
\begin{equation}
\begin{aligned}
Z_{k,t+1}^\to 
&= t - \uptau_k(\ell_{k,t}^\to) + 1 - D_{k,t}(t+1 -  \uptau_k(\ell_{k,t}^\to)) \\
&= (1-D_{k,t})(t - \uptau_k(\ell_{k,t}^\to) + 1) = (1-D_{k,t})(Z_{k,t}^\to + 1).
\end{aligned}
\end{equation}

\textit{Case 3}: $\uptau_k(\ell_{k,t}^\to) < \uptau_k(\ell_{k,t+1}^\to) \leq t$. Then Lemma~\ref{lemma:head_of_line_age} gives
\begin{equation}
Z_{k,t+1}^\to = Z_{k,t}^\to + 1 - D_{k,t}\Delta \uptau_k(\ell_{k,t}^\to).
\end{equation}
Since $(1-D_{k,t})(Z_{k,t}^\to + 1) \leq \left[ Z_{k,t}^\to + 1 - D_{k,t}\Delta \uptau_k(\ell_{k,t}^\to) \right]^+$, the result \eqref{eq:hola_next_step_upper_bound} holds in all three cases.     

The drift bound proceeds from \eqref{eq:hola_next_step_upper_bound}:
\final{
\begin{equation}
\label{eq:initial_drift_bound}
\begin{aligned}
\Delta L_t &\leq \sum_{k=1}^K \frac{(\chi_k + \varepsilon)}{2\overline{x}_k} \left[ \left( Z_{k,t}^\to + 1 - D_{k,t} \Delta \uptau_k(\ell_{k,t}^\to) \right)^2 - \left(Z_{k,t}^\to \right)^2 \right] \\
&= \sum_{k=1}^K \frac{(\chi_k + \varepsilon)}{\overline{x}_k}Z_{k,t}^\to - \sum_{k=1}^K \frac{(\chi_k + \varepsilon)}{\overline{x}_k}D_{k,t} \Delta \uptau_k(\ell_{k,t}^\to) Z_{k,t}^\to \\
&\phantom{\,=\,}+ \sum_{k=1}^K  \frac{(\chi_k + \varepsilon)}{2\overline{x}_k} \left(1-D_{k,t} \Delta \uptau_k(\ell_{k,t}^\to)\right)^2. \\
\end{aligned}
\end{equation}
Now note that whenever $Z_{k,t}^\to > 0$, we have $\uptau_k(\ell_{k,t}^\to) < t$ (see \eqref{eq:age_def} and \eqref{eq:head_of_line_age}), and therefore $D_{k,t} = R_{k}(S_t, A_t)$ according to definition \eqref{eq:queue_departure}. Then the final bound is obtained by substituting $D_{k,t} = R_{k}(S_t, A_t)$ for all $k$ in the second sum, and bounding $\left(1 - D_{k,t} \Delta \uptau_k(\ell_{k,t}^\to) \right)^2 \leq  1 + \left[\Delta \uptau_k(\ell_{k,t}^\to)\right]^2$ for all $k$ in the third sum.
}
\end{proof}

\section{Proof of Theorem~\ref{theorem:regret}}

\begin{proof}
We derive the \textit{$\varepsilon$-tight static policy} $\sigma^\varepsilon$ defined as the following convex combination of the optimal static policy $\sigma^*$ and the $\gamma$-tight static policy $\sigma^\gamma$:
\final{
\begin{equation}
\label{eq:epsilon_tight_static_policy}
\sigma^\varepsilon(a) \triangleq \left( 1 - \frac{\varepsilon}{\gamma} \right) \sigma^*(a) + \frac{\varepsilon}{\gamma} \sigma^\gamma(a) \quad \forall\,a \in  \mathcal{A}.
\end{equation}
}
Let $A_t^\varepsilon$ denote the action played under policy $\sigma^\varepsilon$ in timeslot $t$. By the linearity of expectation and since 
\final{
$\E\mleft[ R_{k}(S_t,A_t^\sigma) \mright] = \E\mleft[ R_{k}(S_1,A_1^\sigma) \mright]$ 
}
for all $t$ under any static policy $\sigma$, we have 
\final{
\begin{equation}
\label{eq:regret_of_epsilon_tight}
\mathrm{Reg}^{\sigma_\varepsilon}_T = \frac{\varepsilon T}{\gamma} \sum_{k=1}^K \E\mleft[ R_{k}(S_1,A_1^*) - R_{k}(S_1,A_1^\gamma) \mright]  \leq \frac{\varepsilon  K T}{\gamma}.
\end{equation}
}

Define the $\varepsilon$-tight regret under a policy $\pi$ as 
\final{
\begin{equation}
\varepsilon\text{-}\mathrm{Reg}^\pi_T \triangleq \sum_{k=1}^K \sum_{t=1}^T \E\mleft[ R_{k}(S_t, A_t^{\varepsilon}) - R_{k}(S_t, A_t^\pi) \mright].
\end{equation}
}
Then from \eqref{eq:regret_of_epsilon_tight}, we have
\begin{equation}
\label{eq:overall_regret_from_epsilon_tight}
\mathrm{Reg}_T^\pi =  \varepsilon\text{-}\mathrm{Reg}^\pi_T + \mathrm{Reg}^{\sigma_\varepsilon}_T \leq \varepsilon\text{-}\mathrm{Reg}^\pi_T +  \frac{\varepsilon  K T}{\gamma}.
\end{equation}
It remains to bound $\varepsilon\text{-}\mathrm{Reg}^\pi_T$ where $\pi$ is the age-based learning policy (we omit the $\pi$ superscript going forward). 
\final{
Observe that for any arm $k$ and timeslot $t$, we have
\begin{equation}
\begin{aligned}
&\overline{x}_k\E\mleft[I_k(A_t^\varepsilon) \mright] = \overline{x}_k \sum_{a \in \mathcal{A}} \sigma^\varepsilon(a)\, I_k(a) \\
&\overset{(a)}{=} \overline{x}_k \sum_{a \in \mathcal{A}} \left[ \left( 1 - \frac{\varepsilon}{\gamma} \right) \sigma^*(a) + \frac{\varepsilon}{\gamma} \sigma^\gamma(a) \right] I_k(a) \\
&=  \left( 1 - \frac{\varepsilon}{\gamma} \right) \overline{x}_k \E\mleft[I_k(A_t^*) \mright] + \frac{\varepsilon}{\gamma} \overline{x}_k \E\mleft[I_k(A_t^\gamma) \mright] \\
&\overset{(b)}{\geq}   \left( 1 - \frac{\varepsilon}{\gamma} \right) \chi_k + \frac{\varepsilon}{\gamma} (\chi_k + \gamma) = \chi_k + \varepsilon \\
\end{aligned}
\end{equation}
where $(a)$ is by the definition \eqref{eq:epsilon_tight_static_policy} of the $\varepsilon$-tight static policy,  $(b)$ is because $\overline{x}_k \E\mleft[I_k(A_t^*) \mright] = \E\mleft[ R_k(S_t, A_t^*) \mright] \geq \chi_k$ according to Lemma~\ref{lemma:existence_of_static_policy} and because $\overline{x}_k\E\mleft[I_k(A_t^\gamma) \mright] \geq \chi_k + \gamma$ according to \eqref{eq:gamma_tight_static_policy}. 
}
\final{
We use the above to bound the expected value of $\mathtt{t1}$ from Lemma~\ref{lemma:elementary_drift_bound} as 
\begin{equation}
\begin{aligned}
\E\mleft[ \mathtt{t1} \mright] &= \sum_{k=1}^K \frac{(\chi_k + \varepsilon)}{\overline{x}_k}\E\mleft[  Z_{k,t}^\to \mright] \\
&\leq \sum_{k=1}^K \E\mleft[I_k(A_t^\varepsilon) \mright] \E\mleft[  Z_{k,t}^\to \mright] = \sum_{k=1}^K \E\mleft[I_k(A_t^\varepsilon)   Z_{k,t}^\to \mright]
\end{aligned}
\end{equation}
where the last step is because $A_t^\varepsilon$ and $Z_{k,t}^\to$ are independent.
}

Recall from Lemma~\ref{lemma:elementary_drift_bound} that $\Delta L_t \leq \mathtt{t1} -  \mathtt{t2} +  \mathtt{t3}$. Combining the above bound on $\E\mleft[ \mathtt{t1} \mright]$ with the previously established bounds on $\E\mleft[ \mathtt{t2} \mright]$ (take expectations on \eqref{eq:t2_conditional_expectation_lower_bound}) and $\E\mleft[ \mathtt{t3} \mright]$ (take expectations on \eqref{eq:t3_conditional_expectation_upper_bound}) gives
\final{
\begin{equation}
\begin{aligned}
\E\mleft[\Delta L_t \mright]
&\leq  \E\mleft[\sum_{k=1}^K I_k(A_t^\varepsilon)  Z_{k,t}^\to - \sum_{k=1}^K I_k(A_t) Z_{k,t}^\to  \mright] + \sum_{k=1}^K \frac{1}{\overline{x}_k\chi_k}.  \\
\end{aligned}
\end{equation}
}
According to the definition of the algorithm \eqref{eq:age_based_algorithm_def}, 
\final{
the quantity $\E\mleft[ \sum_{k=1}^K \left( \eta U_{k,t} + Z_{k,t}^\to\right)I_k(A_t) \mright] -  \E\mleft[\sum_{k=1}^K \left( \eta U_{k,t} + Z_{k,t}^\to\right)I_k(A_t^\varepsilon) \mright]$ is nonnegative. Then adding this quantity to the right-hand side of the previous inequality gives
\begin{equation}
\begin{aligned}
\E\mleft[\Delta L_t \mright]
&\leq \eta\sum_{k=1}^K \E\mleft[ I_k(A_t)  U_{k,t} -  \sum_{k=1}^K   I_k(A_t^\varepsilon) U_{k,t} \mright] +  \sum_{k=1}^K \frac{1}{\overline{x}_k\chi_k}. \\
\end{aligned}
\end{equation}
}
Diving by $\eta$, summing over $t \in [T]$, adding $\varepsilon\text{-}\mathrm{Reg}_T$ to both sides and noticing that $\sum_{t=1}^T\E\mleft[\Delta L_t \mright] = \E\mleft[ L_{T+1}  \mright] \geq 0$ gives 
\final{
\begin{equation}
\begin{aligned}
\varepsilon\text{-}\mathrm{Reg}_T 
&\leq \sum_{t=1}^T\sum_{k=1}^K \E\mleft[ R_{k}(S_t, A_t^\varepsilon) - I_k(A_t^\varepsilon) U_{k,t} \mright] \\
&\phantom{\,=\,}+ \sum_{t=1}^T\sum_{k=1}^K \E\mleft[ I_k(A_t) U_{k,t} - R_{k}(S_t, A_t)\mright] \\    
&\phantom{\,=\,}+  \frac{T}{\eta}\sum_{k=1}^K \frac{1}{\overline{x}_k\chi_k}.
\end{aligned}
\end{equation}
}
\final{
It remains to bound the first two terms involving the UCB estimates. Fix the timeslot $t$ and arm $k$. We have
\begin{equation}
\label{eq:first_ucb_regret_term}
\begin{aligned}
&\E\mleft[ R_{k}(S_t, A_t^\varepsilon) - I_k(A_t^\varepsilon) U_{k,t} \mright]
\overset{(a)}{=} \E\mleft[I_k(A_t^\varepsilon) ( \overline{x}_k - U_{k,t}) \mright] \\
&\leq \E\mleft[I_k(A_t^\varepsilon)\1\{ \overline{x}_k > U_{k,t} \} \mright] 
\leq \prob(\overline{x}_k > U_{k,t}) \\
&\overset{(b)}{=} \prob\mleft(\overline{x}_k > \widetilde{x}_{k,t}  + \sqrt{\frac{3\log t}{2 N_{k,t}}}, N_{k,t} > 0 \mright)
\end{aligned}
\end{equation}
where $(a)$ is because $S_t$ is independent from $A_t^\varepsilon$ and $(b)$ is because $U_{k,t} = 1$ when $N_{k,t} = 0$ and we can never have $\overline{x}_k > 1$. We also have
\begin{equation}
\begin{aligned}
&\E\mleft[ I_k(A_t) U_{k,t} - R_{k}(S_t, A_t)\mright] \overset{(a)}{=} \E\mleft[ I_k(A_t) (U_{k,t} - \overline{x}_k)\mright] \\
&\leq \E\mleft[ I_k(A_t) (U_{k,t} - \overline{x}_k) \1\{ N_{k,t} > 0 \}\mright] + \E\mleft[ I_k(A_t) \1\{ N_{k,t} = 0 \}\mright] \\
&\leq \E\mleft[ I_k(A_t) (\widetilde{x}_{k,t} - \overline{x}_k) \1\{ N_{k,t} > 0 \}\mright] + \E\mleft[ I_k(A_t) \1\{ N_{k,t} = 0 \}\mright] \\
&\phantom{\,=\,}+  \E\mleft[ I_k(A_t) \sqrt{\frac{3 \log t}{2 N_{k,t}}} \1\{ N_{k,t} > 0 \}\mright].
\end{aligned}
\end{equation}
where $(a)$ is because $A_t$ and $S_t$ are independent. Define the ``clean'' event as 
\begin{equation}
\mathcal{E}_{k,t} \triangleq \left\{ N_{k,t} > 0, \widetilde{x}_{k,t} - \overline{x}_k \leq \sqrt{\frac{3\log t}{2 N_{k,t}}} \right\}.
\end{equation}
We use this event to bound the $(\widetilde{x}_{k,t} - \overline{x}_k) \1\{ N_{k,t} > 0 \}$ part inside the first expectation in the previous bound as 
\begin{equation}
\begin{aligned}
&(\widetilde{x}_{k,t} - \overline{x}_k) \1\{ N_{k,t} > 0 \} \\
&= \1 \mathcal{E}_{k,t} \sqrt{\frac{3\log t}{2 N_{k,t}}} + (\widetilde{x}_{k,t} - \overline{x}_k) \1\mleft( \mathcal{E}_{k,t}^\mathtt{C} \cap \{ N_{k,t} > 0 \} \mright) \\
&\leq \1\{ N_{k,t} > 0 \} \sqrt{\frac{3\log t}{2 N_{k,t}}} + \1\mleft( \mathcal{E}_{k,t}^\mathtt{C} \cap \{ N_{k,t} > 0 \} \mright).
\end{aligned}
\end{equation}
Using this inequality to further bound the previous inequality gives
\begin{equation}
\label{eq:second_ucb_regret_term}
\begin{aligned}
&\E\mleft[ I_k(A_t) U_{k,t} - R_{k}(S_t, A_t)\mright] \\
&\leq \prob\mleft( N_{k,t} > 0, \widetilde{x}_{k,t} - \overline{x}_k > \sqrt{\frac{3\log t}{2 N_{k,t}}} \mright) \\
&\phantom{\,=\,}+  \E\mleft[ I_k(A_t) \1\{ N_{k,t} = 0 \}\mright] + \E\mleft[ I_k(A_t) \sqrt{\frac{6 \log t}{ N_{k,t}}} \1\{ N_{k,t} > 0 \}\mright] \\
&\leq \prob\mleft( N_{k,t} > 0, \widetilde{x}_{k,t} - \overline{x}_k > \sqrt{\frac{3\log t}{2 N_{k,t}}} \mright)  + \E\mleft[ I_k(A_t) \sqrt{\frac{6 \log t}{ N_{k,t} \lor 1 }} \mright]
\end{aligned}
\end{equation}
Adding \eqref{eq:first_ucb_regret_term} and \eqref{eq:second_ucb_regret_term} and summing over rounds and arms gives
\begin{equation}
\begin{aligned}
&\sum_{k=1}^K \sum_{t=1}^T \left(\begin{aligned}&\E\mleft[ R_{k}(S_t, A_t^\varepsilon) - I_k(A_t^\varepsilon) U_{k,t} \mright] \\
&+ \E\mleft[ I_k(A_t) U_{k,t} - R_{k}(S_t, A_t)\mright] \end{aligned}\right)  \\
&\leq  \sum_{k=1}^K \sum_{t=1}^T \prob\mleft( N_{k,t} > 0, |\widetilde{x}_{k,t} - \overline{x}_k| > \sqrt{\frac{3\log t}{2 N_{k,t}}} \mright) \\
&\phantom{\,=\,}+  \sum_{k=1}^K \sum_{t=1}^T \E\mleft[ I_k(A_t) \sqrt{\frac{6 \log t}{ N_{k,t} \lor 1 }} \mright].
\end{aligned}
\end{equation}
The first term is bounded by
\begin{equation}
\begin{aligned}
&\sum_{k=1}^K \sum_{t=1}^T \prob\mleft( N_{k,t} > 0, |\widetilde{x}_{k,t} - \overline{x}_k| > \sqrt{\frac{3\log t}{2 N_{k,t}}} \mright) \\
&\overset{(a)}{\leq} \sum_{k=1}^K \sum_{t=1}^T \sum_{n=1}^t \frac{2}{t^3} = \sum_{k=1}^K \sum_{t=1}^T \frac{2}{t^2} \leq \sum_{k=1}^K \frac{\pi^2}{3} = \frac{K\pi^2}{3}  \\
\end{aligned}
\end{equation}
where $(a)$ takes the union bound over the events ${\{N_{n,k} = n\}}$
and bounds each subsequent probability using Hoeffding's inequality. The second term is bounded by 
\begin{equation}
\begin{aligned}
&\sum_{k=1}^K \sum_{t=1}^T \E\mleft[ I_k(A_t) \sqrt{\frac{6 \log t}{ N_{k,t} \lor 1 }} \mright] \\
&\leq \sqrt{6\log T} \sum_{k=1}^K  \E\mleft[ \sum_{t=1}^T I_k(A_t) \sqrt{\frac{1}{ N_{k,t} \lor 1 }} \mright] \\
&\overset{(a)}{\leq} \sqrt{6\log T} \sum_{k=1}^K  \E\mleft[ \1\{ N_{k,T+1} > 0 \} + \sum_{n=1}^{N_{k,T+1}} \sqrt{\frac{1}{ n }} \mright] \\
&\leq \sqrt{6\log T} \sum_{k=1}^K  \E\mleft[ \1\{ N_{k,T+1} > 0 \} + 2\sqrt{N_{k,T+1}} \mright] \\
&\leq \sqrt{56\log T} \E\mleft[ \sum_{k=1}^K \sqrt{N_{k,T+1}} \mright] \\
&\overset{(b)}{\leq}  \sqrt{56\log T} \E\mleft[ \sqrt{K \sum_{k=1}^K N_{k,T+1}} \mright] \\
&\leq \sqrt{56\log T} \sqrt{K I_{\max}T} \\
\end{aligned}
\end{equation}
where $(a)$ is because $N_{k,t+1} = N_{k,t} + 1$ when $I_k(A_t) = 1$, i.e., it increments by 1 every time arm $k$ is played and $N_{k,t} = 0$ the first time it is played, and $(b)$ is by the Cauchy-Schwartz inequality. 
}
\end{proof}